\documentclass{article}

\usepackage[preprint]{neurips_2024}

\usepackage[utf8]{inputenc} 
\usepackage[T1]{fontenc}    
\usepackage{url}            
\usepackage{booktabs}       
\usepackage{amsfonts}       
\usepackage{nicefrac}       
\usepackage{microtype}      
\usepackage{xcolor}         

\usepackage{algorithm}
\usepackage{algorithmic}
\usepackage{graphicx} 
\usepackage{amsmath}
\usepackage{amssymb}
\usepackage{amsthm}
\usepackage{afterpage}
\usepackage{thmtools}
\usepackage{thm-restate}
\usepackage{lastpage}
\usepackage{enumitem}
\usepackage{booktabs}
\usepackage[pagebackref]{hyperref}
\usepackage{adjustbox}
\usepackage{subcaption}

\renewcommand*{\backrefalt}[4]{%
    \ifcase #1 \footnotesize{(Not cited.)}%
    \or        \footnotesize{(Cited on page~#2)}%
    \else      \footnotesize{(Cited on pages~#2)}%
    \fi}

\usepackage{amsmath,amsfonts,bm}

\def\eqref#1{equation~\ref{#1}}

\def\1{\bm{1}}

\DeclareMathAlphabet{\mathsfit}{\encodingdefault}{\sfdefault}{m}{sl}
\SetMathAlphabet{\mathsfit}{bold}{\encodingdefault}{\sfdefault}{bx}{n}

\newcommand{\E}{\mathbb{E}}

\newcommand{\R}{\mathbb{R}}

\newtheorem{theorem}{Theorem}
\newtheorem{lemma}[theorem]{Lemma} 
 
\newtheorem{remark}[theorem]{Remark}
\newtheorem{corollary}[theorem]{Corollary}

\newcommand{\cmp}{x_{*}} 
\newcommand{\w}{w} 

\title{Optimal Linear Decay Learning Rate Schedules \\ and Further Refinements}

\author{%
  Aaron Defazio \\
  Fundamental AI Research Team, Meta \\
  \And
  Ashok Cutkosky \\
  Boston University \\
  \And
  Harsh Mehta \\
  Google Research \\
  \And
  Konstantin Mishchenko \\
  Samsung AI Center \\
}

\begin{document}

\maketitle

\begin{abstract}
Learning rate schedules used in practice bear little resemblance to those recommended by theory.
We close much of this theory/practice gap, and as a consequence are able to derive new \emph{problem-adaptive} learning rate schedules. Our main technical contribution is a refined analysis of learning rate schedules for a wide class of optimization algorithms (including SGD).
When considering only worst-case analysis, our theory predicts that the optimal choice is the \emph{linear decay} schedule where the step-size is set proportional to $1 - t/T$, where $t$ is the current iteration and $T$ is the total number of steps. To go beyond this worst-case analysis, we use the observed gradient norms to derive schedules \emph{refined} for any particular task. 
These refined schedules exhibit learning rate warm-up and rapid learning rate annealing near the end of training. Ours is the first systematic approach to \emph{automatically} yield both of these properties. 
We perform the most comprehensive evaluation of learning rate schedules to date, evaluating across 10 diverse deep learning problems, a series of LLMs, and a suite of logistic regression problems.
We validate that overall, the linear-decay schedule outperforms all commonly used default schedules including cosine annealing. Our adaptive schedule refinement method gives further improvements.
\end{abstract}

\section{Introduction}

For minimizing a function $f$, Stochastic Gradient Descent (SGD) updates the iterate $x_t$ at step $t$ via:
\[
x_{t+1} = x_t - \eta_t g_t,
\]
where $g_t$ is a (possibly stochastic) sub-gradient at $x_t$, and $\eta_t$ is the learning rate (LR) at time $t$. Choosing a sequence of $\eta_t$ for steps $t=1,\dots, T$ is a core problem in optimization.

The learning rate sequence for an optimizer is typically decomposed into two parts: the \textit{baseline} learning rate, indicating the maximum LR to use, and a \textit{schedule}, a sequence that multiplies the baseline LR to give the LR sequence. In this work we focus exclusively on the problem of scheduling. Choosing the right learning rate schedule for best performance is difficult; standard practice is to perform a hyper-parameter sweep over a set of standardized schedules \citep{wu2020wngrad}. 

Setting $\eta_t$ from theory is complicated due to a multitude of potential problem assumptions and the wildly varying schedules that arise from these assumptions. For instance, $\eta_t \propto 1/\sqrt{t}$, $\eta_t \propto 1/t$ and constant schedules $\eta_t=\eta$ are all dictated by three common but different sets of assumptions. Unfortunately, all three work suboptimally in practice for deep learning (Section~\ref{sec:experiments}), and are unpopular in the community \citep{Ge2019}. In this work we focus on two causes for this theory-practice gap:%
\begin{enumerate}[itemsep=0pt,parsep=0pt,topsep=0pt]
    \item Most theory analyzes the average iterate \citep{polyak-avg,ruppert} $\hat{x}_T=\frac{1}{T}\sum_{t=1}^T x_t$ or a randomly sampled iterate. However, in practice the last iterate $x_T$ is used. Analysis of the last iterate is the key to producing practical schedules.
    \item Existing theory for the last iterate often uses crude constant bounds on the gradient norms or curvature. Our new tighter bound involves the entire gradient norm sequence instead, allowing for \textit{problem-adaptive} LR schedules.
\end{enumerate}
\begin{figure*}[t]
\center
\includegraphics[width=\textwidth]{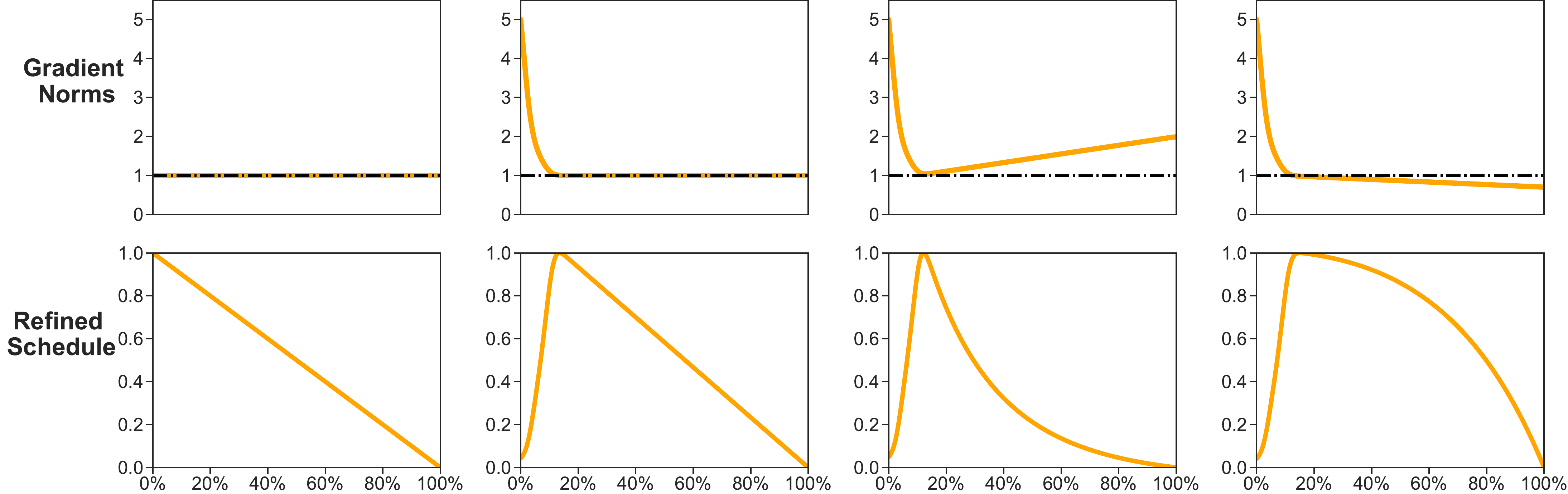}
\caption{\label{fig:example-schedules}Example gradient norm sequences (top row) and the resulting refined schedules given by Algorithm~\ref{alg:refinement} (bottom row). Black dashed line at $y=1$ shown for reference. Percentage of runtime on the x-axis.}
\end{figure*}
\begin{figure}[t]
\begin{algorithm}[H]
\begin{algorithmic}[1]
    \STATE {\bfseries Input:} $G_t = \left\Vert g_t \right \Vert$ sequence of length $T$, smoothing hyper-parameter $\tau > 0$
    \STATE $\hat{G} = \textrm{smoothing\_filter}(G, \textrm{filter\_width}=\tau\,T)$
    \STATE Define $w_{t}=\hat{G}_{t}^{-2}$
    \FOR{$t=1$ {\bfseries to} $T$}
        \STATE \[
    \eta_{t} = w_{t} \sum_{p=t+1}^{T}w_{p}
    \]
    \ENDFOR
    \STATE Return normalized schedule $\eta/\max(\eta)$
\end{algorithmic}
\caption{\label{alg:refinement}Schedule Refinement for SGD}
\end{algorithm}
\end{figure}
Our proposed method is a \emph{refinement} method: it uses a prior training run to produce an improved schedule to use in future runs. The availability of a prior training run is common in practice: a model used in production may need to be regularly retrained on more recent data. The practical variant of our schedule refinement method for SGD is given in Algorithm~\ref{alg:refinement}. Given a sequence of gradient norms produced by a prior run, it outputs a new schedule that is adaptive to the structure of the problem. Mathematically, it is minimizing a novel bound we derive on the function value $f(x_T)$ of the final iterate (Section~\ref{sec:optimize}). The learning rate at each time-step involves a sum of inverse-squared gradient norms from \emph{future} time-steps, a major departure from previous approaches to scheduling.

Our analytical approach also departs from previous approaches by generalizing beyond GD/SGD. Prior analyses of learning rate schedules often rely on the particular form of the iterates to drive the calculations \citep{jain19, zamani2023exact}. Our approach is instead a broad technique that provides learning rate schedules for \emph{any} base optimization algorithm. On a technical level, we design a step-size schedule that converts any algorithm which obtains a vanishing regret into one that ensures a last-iterate guarantee. This means that our refinement technique can provide theoretically-motivated schedules for popular base optimization algorithms like Adam. 
\subsection*{Contributions}
\begin{description}
    \item[Theory] We provide a general analysis of learning rate schedules for arbitrary optimization algorithms. Our approach is the first last iterate analysis to give \emph{exactly} optimal convergence rates for SGD, establishing the linear decay schedule as strongly backed by theory. Our theory suggests an approach for problem-adaptive schedules which we call \emph{refinement}, allowing us to customize the learning rate schedule to each task. 
    \item[Practice] We perform the largest ever comparison of scheduling approaches. Among non-adaptive schedules,\textbf{ we show that warm-up followed by linear decay is the best overall non-adaptive schedule, outperforming cosine decay}. Our refinement method provides significant further improvements, suggesting that our theory provides actionable guidance even for training non-convex neural networks. 
    \item[Theory meets Practice] Our refined schedules exhibit both warmup and decay to zero behavior (Figure \ref{fig:example-schedules}). To our knowledge this is the first time that warmup has arisen \emph{directly from theory} rather than as an empirical heuristic \citep{imagenet1hour}. 
\end{description}
\subsection{Notation}\label{sec:prelim} 
$f\colon\R^d\to \R$ is a convex objective. $x_1,\dots, x_T$ and $z_1,\dots,z_T$ are random vectors in $\R^d$ with $x_1=z_1$, and $\Delta_t \triangleq z_{t+1}-z_t$. $\Delta_t$ will indicate ``baseline'' updates before applying a schedule, and $x_t$ will be iterates after applying the schedule. $g_1,\dots,g_T$ are random vectors in $\R^d$ satisfying $\E[g_t|x_1,\dots,x_t]\in \partial f(x_t)$ ($\E[g_t]=\nabla f(x_t)$ when $f$ is differentiable). $G^2$ is a bound on $\E[\max_t \|g_t\|^2]$. $\w_1,\dots,\w_T$ indicate non-negative random variables in $\R$ such that $g_t$ and $\w_t$ are independent given $x_1,\dots,x_t$. We define $\w_{a:b}\triangleq \sum_{t=a}^b \w_t$. If $a>b$, then $\w_{a:b}\triangleq 0$. $\cmp$ denotes an arbitrary chosen minimizer of $f$. $D\triangleq \|x_1-\cmp\|$ is the ``distance-to-solution'' term, and $f_\star \triangleq f(\cmp)$.

\section{Main Analytical Result}
\label{sec:main-result} 
Our general result is Theorem~\ref{thm:additive}. This allows us to convert any sequence $z_1,\dots,z_T$ with bounded regret $\sum_{t=1}^T \langle g_t,z_t -\cmp\rangle$ into a sequence of iterates $x_1,\dots,x_T$ with a bound on $f(x_T) - f(\cmp)$.
Thus, Theorem~\ref{thm:additive} can be viewed as another addition to the family of reductions from stochastic optimization to regret bounds \citep{cesa2004generalization, cutkosky2019anytime}. 
All proofs can be found in the appendices.

\begin{restatable}{theorem}{thmadditive}\label{thm:additive}
Suppose $z_1,\dots,z_T$ is some arbitrary sequence of vectors. Let $\w_1,\dots,\w_T$ be an arbitrary sequence of non-negative numbers. Recall that we define $\Delta_t = z_{t+1}-z_t$ and $x_1=z_1$. For $t\ge1$, suppose $x_{t+1}$ satisfies:
\begin{align*}
    x_{t+1} = x_t +  \frac{\w_{t+1:T}}{\w_{1:T}} \Delta_t,
\end{align*}
then for any $\cmp$:
\begin{align*}
    \E[f(x_T)-f_*]\le \E\left[\sum_{t=1}^T\frac{1}{\w_{1:T}} \langle \w_t \cdot g_t, z_t - \cmp\rangle \right].
\end{align*}
\end{restatable}
Let us take a moment to consider the implications of this Theorem in the simplest setting of $\w_t=1$ for all $t$. In this case, it is well-known that  by setting $\Delta_t = -\eta g_t$ for $\eta = \frac{D}{G\sqrt{T}}$, one guarantees $\sum_{t=1}^T \langle g_t, z_t - \cmp\rangle \le DG\sqrt{T}$. Thus, we immediately obtain the following important corollary:
\begin{restatable}{corollary}{corsgd}\label{cor:sgd}
Set $x_{t+1} = x_t - \frac{D}{G\sqrt{T}}\left(1-\frac{t}{T}\right) g_t$, then:
\begin{align*}
    \E[f(x_T) - f_*]&\le \frac{DG}{\sqrt{T}}.
\end{align*}
\end{restatable}
The sequence $x_t$ in Corollary~\ref{cor:sgd} is simply stochastic gradient descent ($x_{t+1}=x_t -\eta_t g_t$) equipped with a \emph{linear decay learning rate schedule}:
\begin{equation}
\eta_{t}=\frac{D}{G\sqrt{T}}\left(1-\frac{t}{T}\right). \label{eq:lindec}
\end{equation}
Linear decay emulates the effects of iterate averaging, as the contribution from each gradient to the returned point is approximately the same as it would be in an average; the gradient $g_{T/2}$ appears in half the points in the average, and so its weight is halved. More generally, the gradient $g_{t}$ appears in $T-t$ out of $T$ points and so its weight is $1-t/T$.

This bound has a significantly better constant than previous schedules for last-iterate convergence of SGD in this setting \citep{jain19}, and matches recent work by \cite{zamani2023exact}, who were the first to show that this schedule is actually optimal for (non-stochastic) gradient descent for the Convex $G$-Lipschitz complexity class. Our analysis is more general, and further more, our work is the first to show that in the stochastic case, last-iterate rates can \emph{exactly} match the information-theoretically optimal rate, which normally requires averaging to achieve.

\subsection{Optimizing the bound for data-dependent schedules}\label{sec:optimize}
We have now seen that setting $\w_t=1$ for all $t$ recovers the linear decay schedule, and can obtain the worst-case optimal convergence rates. However, optimizing for the worst case usually yields overly pessimistic behavior. In this section, we build more adaptive schedules that obtain better results on real data. To do this, we simply choose $\w_t$ so as to optimize the bound in Theorem~\ref{thm:additive}.

\begin{restatable}{theorem}{thmoptimal}\label{thm:optimal}
Suppose that $x_{t+1} = x_t - \eta_t g_t$ with $\eta_t = \frac{\w_t \w_{t+1:T}}{\w_{1:T}}$. Then we have:
\begin{align*}
    \E[f(x_T)-f_*]&\le \E\left[\frac{1}{2\cdot \w_{1:T}}\left(D^2 + \sum_{t=1}^T \w_t^2 \|g_t\|^2\right)\right].
\end{align*}
Moreover, for a fixed sequence $\|g_1\|^2,\dots,\|g_T\|^2$, the value of $\frac{1}{2\cdot \w_{1:T}}(D^2 + \sum_{t=1}^T \w_t^2 \|g_t\|^2)$ is minimized by setting:
\begin{align*}
    \w_t = \|g_t\|^{-2} \frac{D}{\sqrt{\sum_{p=1}^T \|g_p\|^{-2}}}.
\end{align*}
\end{restatable}

Theorem~\ref{thm:optimal} suggests that if we knew the sequence of gradient norms ahead of time, then we could optimize the weights $\w_t$ (and therefore the learning rate schedule $\eta_t$) by setting $\w_t\propto \|g_t\|^{-2}$. This yields a practical approach for \emph{refining} a learning rate schedule based on empirical observations. First, perform one run using a baseline schedule to observe the sequence of gradient norms. Then, use these norms to compute an optimal schedule via Theorem~\ref{thm:optimal}. The constant factor $D=\|x_1-u\|$ appearing in the value for $w_t$ plays the role of the ``scale factor'' typically applied to learning rate schedules. A line of recent work has shown that this quantity can be efficiently estimated online without significant downsides \citep{mcmahan2012no, orabona2016coin, cutkosky2018black, mhammedi2020lipschitz, zhang2022pde, carmon2022making, ivgi2023dog, khaled2023dowg, cutkosky2023mechanic}.

The act of changing the schedule from a baseline to a refined data-dependent schedule will change the gradient norm sequence, which may in turn indicate that a  different schedule would have been optimal. Our approach relies on the practical observation that the smoothed gradient norm sequence does not change significantly after refinement. Smoothing the gradient norm sequence is necessary to average over sampling variation. our implementation of Algorithm~\ref{alg:refinement} uses a median smoothing filter on the gradient norms:
\[
\hat{G} = \textrm{med\_filter}(G, \textrm{width}=\tau\,T, \textrm{pad}=(\textrm{nearest}, \textrm{reflect}))
\]
Median filters provide better estimates at the beginning and end of the schedule compared to other smoothing filters.

\subsection{Refined Schedules for Per-Coordinate Optimizers}
Theorem~\ref{thm:optimal} provides an analysis that recovers schedules for SGD. However, practitioners commonly use optimization methods with more complicated \emph{per-coordinate} updates or other preconditioning schemes such as Adam. In Appendix~\ref{sec:percoordinate} we provide an alternative version of Theorem~\ref{thm:optimal} that applies to such algorithms (Theorem~\ref{thm:percoordinate}). This result enables us to develop schedules tailored to any optimization algorithm. However, actually building this schedule in practice may require inspecting the algorithm's internal state, which may be difficult or inefficient. For per-coordinate algorithms like Adam, we suggest simply setting $w_t\propto  1/\|g_t\|_1$ as an approximation (see Section~\ref{sec:l1-norm}).

Figure~\ref{fig:real-schedules} gives the Refined schedules on a set of standard benchmark machine learning problems when initially trained using linear decay schedules to provide the gradient norm sequences. Full details of the models and training setup are in the Appendix. Gradient $\ell_2$ norms are shown for SGD trained problems (ImageNet, RCNN) and $\ell_1$ norms for the Adam trained problems. Our single hyper-parameter $\tau=0.1$ was tuned visually so that the resulting schedules balance smoothness and capturing structure.
\begin{figure*}[t]
\center
\includegraphics[width=0.49\textwidth]{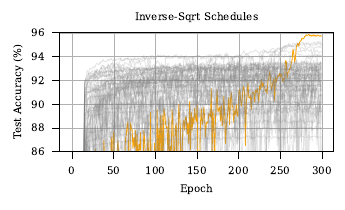}\includegraphics[width=0.49\textwidth]{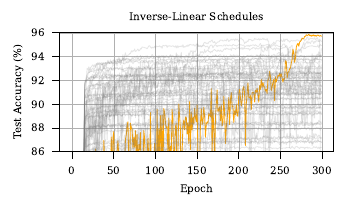}
\vspace{-15pt}
\caption{Training curves on CIFAR-10 for a sweep over inverse-sqrt and inverse-linear hyper-parameters. A linear-decay schedule baseline is shown in orange. All combinations are out-performed by the linear-decay schedule.\label{fig:cifar10_inv_schedules}
}
\end{figure*}
\begin{table*}[t]
\caption{Logistic Regression Experiments (Train Error Rate \%)\label{tab:convex}}
\begin{centering}
\begin{adjustbox}{width=1\textwidth}
\begin{tabular}{ccccccc}
\toprule 
\textbf{Problem}  & \textbf{Stepwise}  & \textbf{Cosine}  & \textbf{Linear}  & \textbf{Refined} $\left\Vert g\right\Vert _{2}^{2}$  & \textbf{Refined} $\left\Vert g\right\Vert _{1}$ & \textbf{Refined}\tabularnewline
\midrule 
Aloi  & $13.38$ {\tiny $\pm0.05$}  & $12.69$ {\tiny $\pm0.05$}  & $12.76$ {\tiny $\pm0.06$}  & $\mathbf{11.75}$ {\tiny$\pm0.01$}  & $12.24$ {\tiny $\pm0.04$} & $12.30$ {\tiny $\pm0.05$}\tabularnewline
 Glass  & $31.39$ {\tiny $\pm0.16$}  & $30.82$ {\tiny $\pm0.22$}  & $30.72$ {\tiny $\pm0.18$}  & $\mathbf{30.05}$ {\tiny$\pm0.44$}  & $\mathbf{29.62}$ {\tiny$\pm0.37$} & $\mathbf{30}$ {\tiny$\pm0.41$}\tabularnewline
Iris  & $\mathbf{1.39}$ {\tiny$\pm0.000$}  & $\mathbf{1.39}$ {\tiny$\pm0.000$}  & $\mathbf{1.39}$ {\tiny$\pm0.000$}  & $\mathbf{1.46}$ {\tiny$\pm0.07$}  & $\mathbf{1.46}$ {\tiny$\pm0.07$} & $\mathbf{1.46}$ {\tiny$\pm0.07$} \tabularnewline
Letter  & $22.24$ {\tiny $\pm0.008$}  & $\mathbf{22.24}$ {\tiny$\pm0.01$}  & $\mathbf{22.20}$ {\tiny$\pm0.02$}  & $\mathbf{22.23}$ {\tiny$\pm0.03$}  & $\mathbf{22.20}$ {\tiny$\pm0.03$} & $\mathbf{22.20}$ {\tiny$\pm0.03$} \tabularnewline
Pendigits  & $4.70$ {\tiny $\pm0.02$}  & $4.67$ {\tiny $\pm0.01$}  & $\mathbf{4.62}$ {\tiny$\pm0.03$}  & $\mathbf{4.56}$ {\tiny$\pm0.02$}  & $\mathbf{4.58}$ {\tiny$\pm0.02$} & $\mathbf{4.56}$ {\tiny$\pm0.04$} \tabularnewline
Sensorless  & $11.84$ {\tiny $\pm0.09$}  & $11.30$ {\tiny $\pm0.09$}  & $11.29$ {\tiny $\pm0.09$}  & $10.71$ {\tiny $\pm0.08$}  & $\mathbf{10.08}$ {\tiny$\pm0.05$} & $\mathbf{10.11}$ {\tiny$\pm0.05$} \tabularnewline
Vehicle  & $18.83$ {\tiny $\pm0.09$}  & $18.49$ {\tiny $\pm0.05$}  & $18.55$ {\tiny $\pm0.06$}  & $\mathbf{18.21}$ {\tiny$\pm0.12$}  & $\mathbf{18.19}$ {\tiny$\pm0.08$} & $\mathbf{18.21}$ {\tiny$\pm0.1$} \tabularnewline
Vowel  & $23.43$ {\tiny $\pm0.08$}  & $22.99$ {\tiny $\pm0.09$}  & $22.94$ {\tiny $\pm0.10$}  & $\mathbf{22.48}$ {\tiny$\pm0.12$}  & $\mathbf{22.44}$ {\tiny$\pm0.08$} & $\mathbf{22.41}$ {\tiny$\pm0.08$} \tabularnewline
\bottomrule 
\end{tabular}
\end{adjustbox}
\par\end{centering}
\end{table*}%
\section{Experiments}
For each problem considered, we first performed a sweep of learning rates on a grid [$10^i$, $2\times 10^i$, $5\times 10^i$] for varying $i$, separately for each schedule. We then ran multiple seeds using the best learning-rate from the grid search. Mean and standard error of the mean was tabulated for each schedule. The best result for each method, up to a statistical significance level of 0.05 using a paired two-sample t-test, is highlighted in bold.
Specific details of the hyper-parameters used are given in Appendix~\ref{sec:experiment_details}. All non-adaptive schedules included a fixed learning-rate warmup, with length following standard practices for the problem. Our step-wise schedule divides by 10 at 30-60-90 percent completion, following standard practices from ImageNet training. The Refined schedules were calculated using the linear schedule to generate the initial gradient norm sequence.
\subsection{Convex problems}
To validate the effectiveness of our method on convex problems, we performed a comparison across 8 commonly used classification problems from the LIBSVM repository, with separable problems excluded (see Section~\ref{sec:limitations}). We used a logistic regression loss together with the Adam optimizer with $\beta=(0.9, 0.95)$, with batch size 16 and trained for 100 epochs. Table~\ref{tab:convex} demonstrates that both the linear decay schedule and our refinement schedule consistently either match or out-perform the cosine schedule. The linear decay schedule matches the cosine schedule on every problem (up to statistical significance), and out-performs it on two problems. \textbf{Overall, the Refined schedule out-performs all previous baseline schedules on the majority of problems}.
\begin{figure*}[t]
\center
\includegraphics[width=\textwidth]{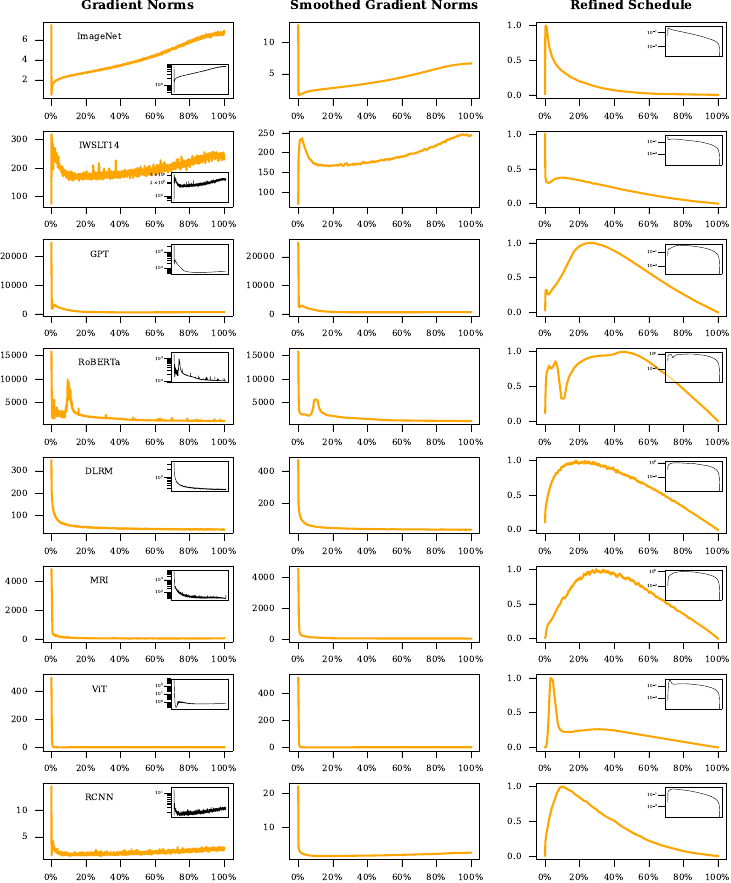}
\caption{\label{fig:real-schedules}Gradient Norm sequences and the resulting Refined schedules, generated using an initial linear decay schedule with warmup for the initial run. Log scale views are inset for scale reference.}
\end{figure*}
\afterpage{\clearpage}
\label{sec:l1-norm}Rather than using the Adam specific weighting for deriving the Refined schedules, it is often convenient to use other norms, particularly if existing logs are available using these other norms. In Table~\ref{tab:convex} we present results using both the $\ell_1$ norm and $\ell_2$ norm-squared weighting. Both are shown to maintain the advantages of the refinement technique, out-performing the non-refined schedules consistently. Given the ease of logging $\ell_1$ norms compared to the internal Adam state, we advocate for weights to be the inverse of the $\ell_1$ norm (no squaring) when using Adam.
\subsection{Deep Learning experiments}
\label{sec:experiments}
 Classical any-time learning rates schedules such as $1/\sqrt{t}$ and $1/t$ are commonly used in theoretical analysis of optimization methods, yet they are rarely used by practitioners. To illustrate why, in Figure~\ref{fig:cifar10_inv_schedules}, we sweep over the learning rate $\eta$ and offset $\beta$ for schedules of the form: 
\begin{equation*}
\eta \frac{\beta}{\beta + t}\quad\text{and}\quad 
\eta \frac{\sqrt{\beta}}{\sqrt{\beta + t}}.
\end{equation*}
A 5\% duration learning rate warmup was also included in each case. \textbf{Even with 60 hyper-parameter combinations tried for each family, neither are able to match the performance of the linear-decay schedule}. For our deep learning experiments we implemented two commonly used practical variants of these schedules, since sweeping the $\beta$ hyper-parameter is computationally impractical. The vanilla version uses $\beta=1$, and the \emph{offset} version uses $\beta$ set to match the duration of the warmup period (a common setting due to its implementation in FairSeq).
\subsection{Large-scale schedule benchmark}
We performed a large-scale comparison of schedules across common deep learning benchmarks, (pre-)training each problem from scratch for each task: 
\begin{description}[itemsep=0pt,parsep=0pt,topsep=3pt]
\item[CIFAR10] For a small-scale image-classification problem, we chose CIFAR10 \citep{cifar}. A high-performance Wide ResNet architecture was used \citep{BMVC2016_87}. Test error percentage is reported.
\item[CIFAR100] A more realistic yet still small scale benchmark. We used a compact DenseNet architecture \citep{densenet} to increase the diversity of model architectures tested. Test error percentage is reported.
\item[ImageNet] We used a ResNet-50 architecture \citep{he2016deep} for the ImageNet \citep{ILSVRC15} classification task. We use a standard data-augmentation pipeline following other recent comparisons in the optimization literature. Test error percentage is reported.
\item[IWSLT14] A small-scale German-English translation task \citep{cettolo2014report} using a LSTM model \citep{wiseman-rush-2016-sequence}. Test Perplexity is reported.
\item[GPT] A modern small (162M parameter) GPT-style auto-regressive transformer model \citep{gpt} trained on the large Book-Wiki corpus \citep{bookcorpus}. Test Perplexity is reported.
\item[RoBERTa] A masked autoencoder variant \citep{liu2019roberta} also trained on the Book-Wiki corpus. Test Perplexity is reported.
\item[ViT] A high-performance ImageNet classifier using a Vision Transformer \citep{dosovitskiy2020image}. Test error percentage is reported.
\item[DLRM] A modern recommendation system engine \citep{DLRM19} trained on the open Criteo Kaggle Display Advertising dataset. Test Accuracy is reported.
\item[MRI] A large stacked U-Net architecture \citep{
sriram2020end} trained on the fastMRI dataset \citep{zbontar2018fastmri}, an image-to-image regression task from the medical imaging domain. Test set metric $100 \cdot (1-\text{SSIM})$ is reported.
\item[RCNN] The object detection method Faster-RCNN \citep{faster_rcnn} trained on the COCO 2017 dataset, using a ResNet-50 ImageNet pretrained backbone. $100 - \text{box AP}$ is reported.
\end{description}
\begin{figure}[t]\label{fig:llm-schedules}
\center
\includegraphics[trim={60pt 200pt 60pt 200pt},clip,width=0.49\linewidth]{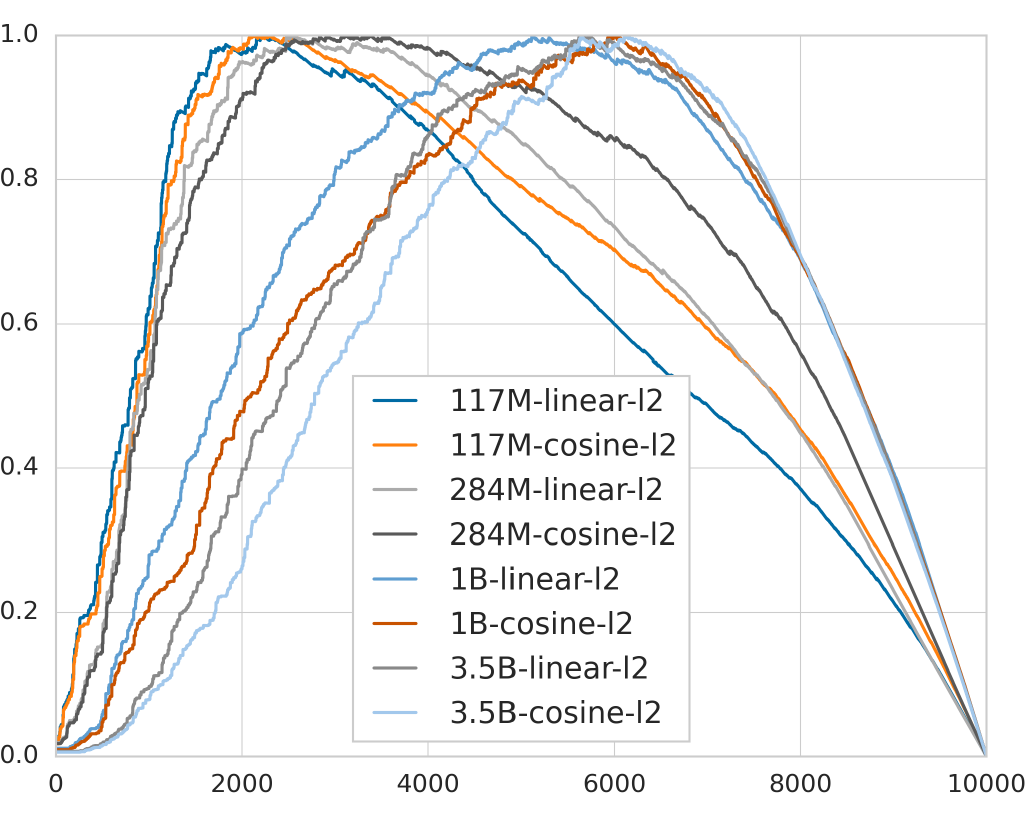} 
\includegraphics[trim={60pt 200pt 60pt 200pt},clip,width=0.49\linewidth]{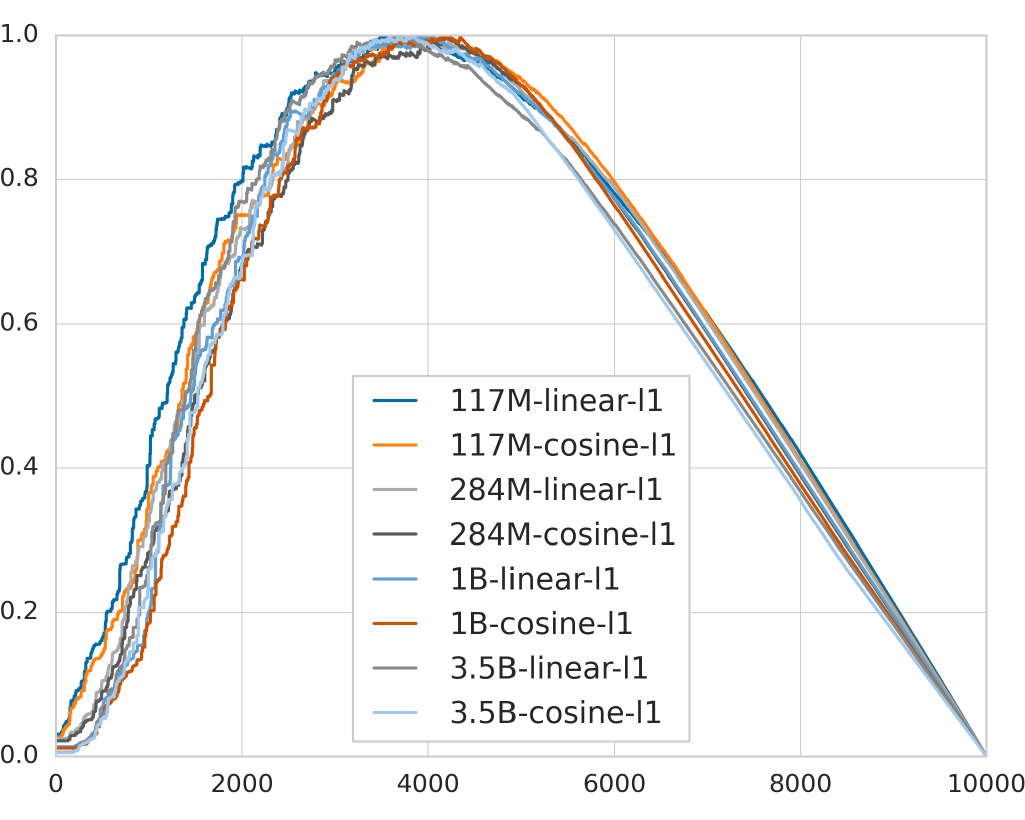} 
\caption{\label{fig:llm_schedules} Normalized learning rate schedules learned on vanilla Transformer-based LLM trained on C4 dataset. Left: schedules obtained when setting $\w_t\propto \|g_t\|^{-2}$. Right: schedules when $w_t\propto  1/\|g_t\|_1$. We find that $\ell_1$ norm is much more consistent across model sizes and baseline schedules. Thus, we used $w_t\propto  1/\|g_t\|_1$ weighting from linear baselines to obtained Refined schedule mentioned in Table \ref{tab:llm}. The x axis is training steps.}
\end{figure}
\begin{table*}[t]
\caption{\label{tab:classical}Classical Schedule Comparison Against the Linear Schedule (lower = better).}
\vspace{-0.5em}
\begin{center}
\begin{adjustbox}{width=1\textwidth}
\begin{tabular}{cccccccc}
\toprule 
\textbf{Problem} & \textbf{Flat} & \textbf{1/t} & \textbf{1/sqrt} & \textbf{Offset 1/t} & \textbf{Offset 1/sqrt} & \textbf{Linear} \tabularnewline
\midrule 
CIFAR10 &  $8.04$ {\tiny $ \pm .13$} & $5.42$ {\tiny $ \pm .28$} & $6.37$ {\tiny $ \pm .41$} & $9.23$ {\tiny $ \pm .08$} & $6.58$ {\tiny $ \pm .06$} & $\mathbf{4.35}$ {\tiny $ \pm .05$} \tabularnewline
CIFAR100 &  $30.43$ {\tiny $ \pm .20$} & $26.58$ {\tiny $ \pm .11$} & $29.09$ {\tiny $ \pm .16$} & $32.80$ {\tiny $ \pm .07$} & $27.62$ {\tiny $ \pm .13$} & $\mathbf{22.11}$ {\tiny $ \pm .08$} \tabularnewline
ImageNet &  $33.00$ {\tiny $ \pm .14$} & $26.48$ {\tiny $ \pm .06$} & $28.35$ {\tiny $ \pm .05$} & $47.94$ {\tiny $ \pm .08$} & $27.34$ {\tiny $ \pm .06$} & $\mathbf{23.11}$ {\tiny $ \pm .07$} \tabularnewline
IWSLT14 &  $8.07$ {\tiny $ \pm .02$} & $7.62$ {\tiny $ \pm .01$} & $7.52$ {\tiny $ \pm .01$} & $12.89$ {\tiny $ \pm .06$} & $8.48$ {\tiny $ \pm .01$} & $\mathbf{7.10}$ {\tiny $ \pm .01$} \tabularnewline
GPT &  $20.20$ {\tiny $ \pm .000$} & $18.99$ {\tiny $ \pm .04$} & $19.48$ {\tiny $ \pm .02$} & $27.85$ {\tiny $ \pm .05$} & $22.88$ {\tiny $ \pm .006$} & $\mathbf{18.60}$ {\tiny $ \pm .02$} \tabularnewline
RoBERTa &  $4.52$ {\tiny $ \pm .005$} & $4.25$ {\tiny $ \pm .007$} & $4.33$ {\tiny $ \pm .01$} & $5.33$ {\tiny $ \pm .02$} & $5.15$ {\tiny $ \pm .02$} & $\mathbf{3.94}$ {\tiny $ \pm .007$} \tabularnewline
DLRM &  $\mathbf{20.95}$ {\tiny $ \pm .006$} & $47.59$ {\tiny $ \pm 6.45$} & $45.99$ {\tiny $\pm 5.98$} & $\mathbf{20.94}$ {\tiny $ \pm .007$} & $20.99$ {\tiny $ \pm .009$} & $\mathbf{20.94}$ {\tiny $ \pm .006$} \tabularnewline
MRI &  $\mathbf{9.00}$ {\tiny $ \pm .04$} & $8.91$ {\tiny $ \pm .01$} & $\mathbf{8.98}$ {\tiny $ \pm .04$} & $9.53$ {\tiny $ \pm .08$} & $9.16$ {\tiny $ \pm .05$} & $\mathbf{8.88}$ {\tiny $ \pm .02$} \tabularnewline
ViT &  $30.11$ {\tiny $ \pm .27$} & $28.36$ {\tiny $ \pm .40$} & $28.53$ {\tiny $ \pm .15$} & $73.84$ {\tiny $ \pm 6.08$} & $50.36$ {\tiny $ \pm 12.39$} & $\mathbf{24.82}$ {\tiny $ \pm .31$} \tabularnewline
RCNN &  $65.43$ {\tiny $ \pm .12$} & $63.38$ {\tiny $ \pm .05$} & $64.13$ {\tiny $ \pm .10$} & $79.32$ {\tiny $ \pm .07$} & $69.25$ {\tiny $ \pm .07$} & $\mathbf{60.98}$ {\tiny $ \pm .02$} \tabularnewline 
\bottomrule 
\end{tabular}
\end{adjustbox}
\end{center}
\end{table*}%
\begin{table*}[t]
\caption{\label{tab:modern}Modern Schedule Comparison (lower = better).}\vspace{-0.5em}
\begin{center}
\begin{tabular}{ccccc}
\toprule 
\textbf{Problem} & \textbf{Stepwise} & \textbf{Cosine} & \textbf{Linear} & \textbf{Refinement} \tabularnewline
\midrule
CIFAR10 &  $4.53$ {\tiny $ \pm .03$} & $\mathbf{4.27}$ {\tiny $ \pm .04$} & $\mathbf{4.35}$ {\tiny $ \pm .05$} & - \tabularnewline
CIFAR100 &  $22.78$ {\tiny $ \pm .10$} & $22.59$ {\tiny $ \pm .09$} & $\mathbf{22.11}$ {\tiny $ \pm .08$} & - \tabularnewline
ImageNet &  $23.51$ {\tiny $ \pm .07$} & $\mathbf{23.10}$ {\tiny $ \pm .06$} & $\mathbf{23.11}$ {\tiny $ \pm .07$} & $\mathbf{23.12}$ {\tiny $ \pm 0.03$} \tabularnewline
IWSLT14 &  $7.43$ {\tiny $ \pm .01$} & $7.17$ {\tiny $ \pm .009$} & $7.10$ {\tiny $ \pm .01$} & $\mathbf{6.92}$ {\tiny $ \pm .03$}\tabularnewline
GPT &  $19.70$ $\pm.03$ & $18.65$ {\tiny $ \pm .02$} & $18.60$ {\tiny $ \pm .02$} & $\mathbf{18.29}$ {\tiny $ \pm .005$}\tabularnewline
RoBERTa &  $4.36$ {\tiny $ \pm .01$} & $4.07$ {\tiny $ \pm .000$} & $3.94$ {\tiny $ \pm .007$} & $\mathbf{3.86}$ {\tiny $ \pm .005$} \tabularnewline
DLRM &  $20.95$ {\tiny $ \pm .008$} & $\mathbf{20.94}$ {\tiny $ \pm .005$} & $\mathbf{20.94}$ {\tiny $ \pm .006$} & $\mathbf{20.94}$ {\tiny $ \pm .009$} \tabularnewline
MRI &  $8.97$ {\tiny $ \pm .02$} & $8.90$ {\tiny $ \pm .03$} & $8.88$ {\tiny $ \pm .02$} & $\mathbf{8.85}$ {\tiny $ \pm .01$} \tabularnewline
ViT & $26.27$ {\tiny $ \pm .33$} & $\mathbf{24.56}$ {\tiny $ \pm .15$} & $24.82$ {\tiny $ \pm .31$} & $25.53$ {\tiny $ \pm .16$} \tabularnewline
RCNN &  $61.76$ {\tiny $ \pm .06$} & $\mathbf{61.00}$ {\tiny $ \pm .04$} & $\mathbf{60.98}$ {\tiny $ \pm .02$} & $61.21$  {\tiny $\pm .03$}\tabularnewline
\bottomrule 
\end{tabular}
\end{center}
\end{table*}%
\begin{table}[ht]
\footnotesize
\begin{subtable}[t]{0.53\linewidth}\centering
\begin{adjustbox}{width=\linewidth}
\begin{tabular}{ccccc}
\toprule
\textbf{Schedule} & \textbf{1 Epoch} & \textbf{5 Epochs} & \textbf{30 Epochs} \\
\midrule 
Cosine & $35.80$ {\footnotesize $ \pm 0.27$} & $15.07$ {\footnotesize $ \pm 0.11$} & $\mathbf{6.08}$ {\footnotesize $ \pm 0.03$} \tabularnewline
Lin. Decay & $\mathbf{33.84}$ {\footnotesize $ \pm 0.17$} & $\mathbf{14.36}$ {\footnotesize $ \pm 0.07$} & $\mathbf{6.12}$ {\footnotesize $ \pm 0.08$} \tabularnewline
Refined & $\mathbf{34.04}$ {\footnotesize $ \pm 0.19$} & $\mathbf{14.59}$ {\footnotesize $ \pm 0.08$} & $6.24$ {\footnotesize $ \pm 0.06$} \tabularnewline
\bottomrule
\end{tabular}
\end{adjustbox}
\caption{CIFAR10 Training Time Ablations (Test Error \%)\label{tab:cifar-short}}
\end{subtable}
\begin{subtable}[t]{0.47\linewidth}\centering
\begin{adjustbox}{width=\linewidth}
\begin{tabular}{ccccc}
\toprule
\textbf{Schedule} & \textbf{117M} & \textbf{284M} & \textbf{1B} & \textbf{3.5B} \\
\midrule
Cosine & $3.089$ & $2.891$ & $2.729$ & $2.631$ \tabularnewline
Linear Decay & $3.087$ & $2.888$ & $2.725$ & $\mathbf{2.625}$ \tabularnewline
Refined & $\mathbf{3.075}$ & $\mathbf{2.884}$ & $\mathbf{2.722}$ & $2.634$ \tabularnewline
\bottomrule
\end{tabular}
\end{adjustbox}
\caption{LLM Ablations (C4 validation loss)\label{tab:llm}}
\end{subtable}
\end{table}%

Tables \ref{tab:classical} \& \ref{tab:modern} show the results. We break the schedules into two categories, classical and modern. \textbf{The modern schedules consistently outperform the classical schedules, often by large margins}. Although this is common folk-law in deep learning, we are not aware of any existing large-scale experiments establishing this. Our comparison of modern schedules shows a clear hierarchy among the schedules. The Stepwise schedule is dominated by the Cosine schedule, and \textbf{the Linear schedule matches or outperforms the Cosine schedule on all problems except ViT}. The refined schedule further outperforms the Linear schedule on 5 of the 10 problems, but shows mixed results on ViT and RCNN. The refinement process produced degenerate schedules that fail on the CIFAR problems, we discuss this in  Appendix~\ref{sec:limitations}.
\subsection{Cosine schedule ablations}
To further investigate the failure modes of the cosine schedule, we ran a series of shorter duration training runs on CIFAR10. Our premise was that the cosine schedule is heavily over-fit to long training duration computer vision problems.
As shown in Table~\ref{tab:cifar-short}, the cosine schedule begins to under-perform the linear decay schedule and the refined schedules when training for less than 30 epochs. In contrast, while the refined schedule also under-performs for longer duration training it has no statistically significant differences from the linear decay schedule for shorter duration training, where they both perform consistently better than the cosine schedule.

\subsection{Large Language Model size ablations}
In addition to the results above, we validate our insights on Language Models by performing an additional set of experiments directly comparing to the setting explored in Chinchilla \citep{hoffmann2022an} where we train a vanilla Transformer model on the C4 dataset \citep{2020t5} for 10k steps with a token batch size of around 524k ($2^{19}$) and learning rate $0.0002$. Starting with \citet{hoffmann2022an}, most recent LLMs employ the AdamW optimizer with Linear Warmup and Cosine Decay. We perform a head-to-head comparison of Cosine decay, Linear decay and Refined schedules (the latter is refined from Linear decay). As shown in Table \ref{tab:llm}, \textbf{contrary to popular wisdom, linear decay performs better than cosine decay across all model sizes}. The Refined schedule further outperforms linear decay for all but the 3.5B sized model.

\section{Discussion}
Although our theoretical analysis assumes that the underlying optimization problems are convex, as we have demonstrated, our method appears to work well for non-convex deep learning problems in practice. A similar theoretical analysis in the non-convex setting can not be done as last iterate rates are provably impossible in that setting \citep{pmlr-v119-drori20a}.

Gradient sequences that slowly decrease or slowly increase after stabilizing are often observed in practice (Figure~\ref{fig:real-schedules}), and the resulting schedules that our framework produces show interesting behavior, resembling \textit{polynomial-decay} schedules with exponents $p$: 
\begin{equation}
\eta_{t} \propto \left(1-\frac{t}{T}\right)^p. \label{eq:polysched}
\end{equation}
Schedules of this form are already in use, with $p$ typically tuned as a hyper-parameter. 
From Figure~\ref{fig:example-schedules}, we see that $p<1$ should be used when the gradient norm is decreasing over time, $p>1$ when the gradient norm is increasing, and $p=1$ when the gradient sequence is flat. We use the term \emph{polynomial decay} for the schedule given in Equation~\ref{eq:polysched}, which its name in both PyTorch and Tensorflow, although in the literature schedules of the form $1/t^\alpha$ are also sometimes referred to as polynomial schedules.

\section{Related Work}
The standard schedule for strongly-convex stochastic optimization is $\eta_t \propto 1/t$. This is optimal when appropriate averaging is used \citep{shamir-zhang, lacostejulien2012simpler, rakhlin}. Without strongly convexity, theory typically recommends a constant schedule.

Learning rate schedules used in applications usually fall within a few standard classes supported by major frameworks. To survey the use of schedulers in open source software, we performed a GitHub search on each of the schedulers in the PyTorch Library. Table~\ref{table:sched-popularity} (Appendix) shows that the polynomial decay schedule, which includes the linear decay schedule as a special case, is currently the least popular scheduler in PyTorch. Step-wise schedules are by far the most popular.

Our approach is motivated by the remarkable work of \citet{eigencurve}, who show that for the last iterate of quadratic problems, minimizing an upper bound can yield improved schedules. However their work requires knowledge of the full Eigenspectrum of the Hessian, making it impractical to use. Our theory relies the sequence of expected gradient norms, a more tractable quantity.

Ours is not the first reduction from regret bounds to a last-iterate guarantee: \cite{cutkosky2019anytime} suggests evaluating gradients at running averages to achieve the same convergence rate as Theorem~\ref{thm:additive}. However, our method is directly interpretable in terms of learning rate schedules to yield immediate practical guidance while the previous technique is a more opaque transformation of the base algorithm. Nevertheless, understanding the differences is a valuable future direction.

There has been significant work on learning rate adaptation using gradient or loss information. The AdaGradNorm-type schedule is backed by strong theory \citep{wu2020wngrad, lessregret, adagrad, ward2019adagrad, li2019convergence}. They suggest $
\eta_{t}\propto 1/ \sqrt{\sum_{t=1}^{T}\left\Vert g_{t}\right\Vert ^{2}}$, but it is unclear how to apply these schedules to other optimizers such as Adam.

Methods based on hyper-gradients (i.e.\ gradients of the learning rate) are another well-studied class of adaptive method. They use recent gradient evaluations and other local information to update the learning rate \citep{almeida99,franceschi17a, bengio-hyper, domkehyper, pedregosa-hyper, baydin2018hypergradient, fhutter-hyper, donini2020, ultimategd}. Alternatively, hyper-gradients can be used to signal when to decrease the learning rate \citep{pflug83, pflug88, lang2019, zhang2020statistical}. 
 Methods based on loss minimization choose a learning rate that greedily minimizes the test or train loss \citep{xu2019learning, autolr2021, kim2021automated}.


\bibliography{scheduling}
\bibliographystyle{apalike}

\appendix

\newpage
\section{Limitations}
\label{sec:limitations}
There are certain problems like CIFAR-10 and CIFAR-100 where we can clearly observe potential pitfalls of our approach (Figure~\ref{fig:cifar-schedules}) when our assumptions break down. On these problems $100\%$ train accuracy is reached near the end, driving the gradient norm sequence to zero. The resulting refined schedule rapidly increases the step size near the end of optimization, which if used results in training divergence. We suggest using a linear decay schedule instead on problems where the gradient norm drops significantly at the end.

Our work assumes that multiple training runs for the given workload are possible. This often holds in practice, particularly in industry settings where the same models are repeatedly trained on fresh data. Even in cases where large one-off training runs are performed, refined schedules can be extrapolated from smaller trial runs, or similar past runs.

Although our theory is limited to the convex setting, it seems to have predictive power for non-convex problems also. The goal of our theory is to guide practice, not to provide an explanation of potential pathological worst case behavior that can occur in non-convex settings. For this purpose, simplified settings are ideal as they allow for precise and predictive results.

\begin{figure}[t]
\center
\includegraphics[width=\textwidth]{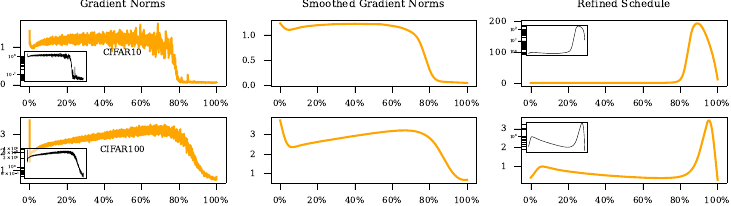}
\caption{\label{fig:cifar-schedules} Limitations of refinement: if model over-fits, our method produces degenerate schedules.}
\end{figure}
\section{All-Tail Summation Bound}\label{sec:proof:tail}
In this section we develop some key Lemmas that underlie our results. The starting point of this development is the following result from \cite{Orabona_2020, lin2016}. We present the result and proof below for completeness, and then provide an improved version of the bound (Lemma~\ref{lem:tail}) that is used to derive our analytical results.
\subsection{Existing bound}
\begin{lemma}\label{lem:tail_initial}
Let $q_{t}\geq0$, and let $\eta_{t}$ be a positive non-increasing
sequence. Then:
\[
\eta_{T}q_{T}\leq\frac{1}{T}\sum_{t=1}^{T}\eta_{t}q_{t}+\sum_{k=1}^{T-1}\frac{1}{k(k+1)}\sum_{t=k+1}^{T}\eta_{t}\left(q_{t}-q_{k}\right),
\]
\end{lemma}

\begin{proof}
We present below a version of the proof that's written to move the
single inequality application in the whole proof to the last step.
Define $S_{k}=\frac{1}{k}\sum_{t=T-k+1}^{T}\eta_{t}q_{t}$.

Then
\begin{align*}
kS_{k} & =(k+1)S_{k+1}-\eta_{T-k}q_{T-k}\\
 & =kS_{k+1}+S_{k+1}-\eta_{T-k}q_{T-k}\\
 & =kS_{k+1}+\frac{1}{k+1}\sum_{t=T-k}^{T}\left(\eta_{t}q_{t}-\eta_{T-k}q_{T-k}\right).
\end{align*}

Dividing through by $k$:
\[
S_{k}=S_{k+1}+\frac{1}{k(k+1)}\sum_{t=T-k}^{T}\left(\eta_{t}q_{t}-\eta_{T-k}q_{T-k}\right).
\]

Unrolling, we have:
\[
S_{1}=S_{T}+\sum_{k=1}^{T-1}\frac{1}{k(k+1)}\sum_{t=T-k}^{T}\left(\eta_{t}q_{t}-\eta_{T-k}q_{T-k}\right).
\]

Now we use $S_{1}=\eta_{T}q_{T}$. Note that the first entry in that
sum is zero so we may shift the indexing to start at $t=T-k+1$. Giving:
\begin{align*}
\eta_{T}q_{T} & =S_{T}+\sum_{k=1}^{T-1}\frac{1}{k(k+1)}\sum_{t=T-k+1}^{T}\left(\eta_{t}q_{t}-\eta_{T-k}q_{T-k}\right)\\
 & \leq S_{T}+\sum_{k=1}^{T-1}\frac{1}{k(k+1)}\sum_{t=T-k+1}^{T}\eta_{t}\left(q_{t}-q_{T-k}\right).
\end{align*}

Where the final step uses the fact that $\eta_{t}$ is decreasing.
Plugging in the definition of $S_{T}$, and substituting $k=T-k$ to simplify gives the result.
\end{proof}

\subsection{Improved expression}
\begin{lemma}\label{lem:tail}
Let $q_{t}$ be any sequence, and let $\w_t$ be a positive sequence.
Then:
\begin{align*}
q_{T}&=\frac{1}{\sum_{t=1}^{T}\w_{t}}\sum_{t=1}^{T}\w_{t}q_{t}+\sum_{k=1}^{T-1}\frac{\w_{k}}{\sum_{t=k+1}^{T}\w_{t}}\left(\frac{1}{\sum_{t=k}^{T}\w_{t}}\sum_{t=k}^{T}\w_{t}\left(q_{t}-q_{k}\right)\right)\\
&=\frac{1}{\w_{1:T}}\sum_{t=1}^{T}\w_{t}q_{t}+\sum_{k=1}^{T-1}\left(\frac{1}{\w_{k+1:T}} - \frac{1}{\w_{k:T}}\right)\sum_{t=k}^{T}\w_{t}\left(q_{t}-q_{k}\right).
\end{align*}
\end{lemma}

\begin{proof}
Define
\[
S_{k}=\frac{1}{\sum_{t=T-k+1}^{T}\w_{t}}\sum_{t=T-k+1}^{T}\w_{t}q_{t}.
\]
Note that with this definition:
\[
S_{1}=\frac{1}{\sum_{t=T}^{T}\w_{t}}\sum_{t=T}^{T}\w_{t}q_{t}=q_{T},
\]
and $S_{T}$ is the full sum:
\[
S_{T}=\frac{1}{\sum_{t=1}^{T}\w_{t}}\sum_{t=1}^{T}\w_{t}q_{t}.
\]
The difference from the weighting used in the the Lemma above: we
normalized by the sum of the step sizes rather than $k$. We get the
following expansion: 

\begin{align*}
\left(\sum_{t=T-k+1}^{T}\w_{t}\right)S_{k} & =\sum_{t=T-k+1}^{T}\w_{t}q_{t}\\
 & =\sum_{t=T-k}^{T}\w_{t}q_{t}-\w_{T-k}q_{T-k}\\
 & =\left(\sum_{t=T-k}^{T}\w_{t}\right)S_{k+1}-\w_{T-k}q_{T-k}\\
 & =\left(\sum_{t=T-k+1}^{T}\w_{t}\right)S_{k+1}+\left(\w_{T-k}S_{k+1}-\w_{T-k}q_{T-k}\right).
\end{align*}
So dividing through by $\sum_{t=T-k+1}^{T}\w_{t}$:
\[
S_{k}=S_{k+1}+\frac{\w_{T-k}}{\sum_{t=T-k+1}^{T}\w_{t}}\left(S_{k+1}-q_{T-k}\right).
\]
Unrolling
\[
S_{1}=S_{T}+\sum_{k=1}^{T-1}\frac{\w_{T-k}}{\sum_{t=T-k+1}^{T}\w_{t}}\left(S_{k+1}-q_{T-k}\right).
\]
Note that, plugging in $S_{k+1}$:

\begin{align*}
S_{k+1}-q_{T-k} & =\frac{1}{\sum_{t=T-k}^{T}\w_{t}}\sum_{t=T-k}^{T}\w_{t}q_{t}-q_{T-k}\\
 & =\frac{1}{\sum_{t=T-k}^{T}\w_{t}}\sum_{t=T-k}^{T}\w_{t}\left(q_{t}-q_{T-k}\right).
\end{align*}
So we have:
\[
q_{T}=\frac{1}{\sum_{t=1}^{T}\w_{t}}\sum_{t=1}^{T}\w_{t}q_{t}+\sum_{k=1}^{T-1}\frac{\w_{T-k}}{\sum_{t=T-k+1}^{T}\w_{t}}\left(\frac{1}{\sum_{t=T-k}^{T}\w_{t}}\sum_{t=T-k}^{T}\w_{t}\left(q_{t}-q_{T-k}\right)\right).
\]
Finally, we make the simple change of variables $k\to T-k$ to yield the result.
\end{proof}
\begin{remark}
We will typically use this result by setting $q_t = f(x_t) - f_{*}$ in order to bound $q_{T}=f(x_{T})-f_{*}$.
By using a weighting sequence $\w_{t}$ elsewhere, we are able to
remove the $\w_{T}$ weight from in front of the $q_{T}$ term (in contrast to Lemma~\ref{lem:tail_initial}).
This is crucial, as we want to be able to analyze the situation in which $\w_t$ drops extremely small near the end,
and yet still bound $q_{T}=f(x_{T})-f_{*}$, but if $q_{T}$ is weighted
by $\w_{T}$ we will get very loose bounds when $\w_{T}$ is small.
Notice also that we have an equality instead of an inequality, and
we have not had to impose the requirement that $\w$ be a non-increasing
sequence.
\end{remark}

\section{Proof of Theorem ~\ref{thm:additive}}\label{sec:proof:thmadditive}

Before proving Theorem~\ref{thm:additive}, we need the following important Lemma:

\begin{restatable}{lemma}{lemrearrange}\label{lem:rearrange}
    Suppose $z_1,\dots,z_T$ is an arbitrary sequence of vectors. Let $\w_1,\dots,\w_T$ be an arbitrary sequence of positive numbers. Define the sequence  $x_1,\dots,x_T$ recursively by $x_1=z_1$ and:
\begin{align*}
    x_t = \w_{t:T}\left( \frac{z_t}{\w_{1:T}} + \sum_{p=1}^{t-1} x_p\left(\frac{1}{\w_{p+1:T}}-\frac{1}{\w_{p:T}}\right)\right).
\end{align*}
Suppose $g_t$ are random variables with $\E[g_t|x_1,\dots,x_t] \in \partial f(x_t)$ for some convex $f$. Then:
\begin{align*}
\E[f(x_T) - f(\cmp)] &\le \E\left[\frac{1}{\w_{1:T}}\sum_{t=1}^T\w_t\langle g_t, z_t - \cmp\rangle \right].
\end{align*}
\end{restatable}
\begin{proof}
Let us set $q_t = f(x_t) - f(\cmp)$. Then we have $\E[q_t] \le \E[\langle g_t , x_t -\cmp\rangle]$ and $\E[q_t -q_k]\le \E[\langle g_t, x_t - x_k\rangle]$. Then, Lemma~\ref{lem:tail} implies:
\begin{align*}
    \E[q_T]&\le \E\left[\frac{1}{\w_{1:T}} \sum_{t=1}^T \w_t \langle g_t, x_t - \cmp\rangle + \sum_{k=1}^{T-1}\left(\frac{1}{\w_{k+1:T}} - \frac{1}{\w_{k:T}}\right) \sum_{t=k}^T \w_t \langle g_t, x_t - x_{k}\rangle \right].
\end{align*}
Now, let us find the coefficient of $\langle g_t, x_t\rangle$ in the above expression. This is:
\begin{align*}
    &\frac{\w_t}{\w_{1:T}} +\left[ \sum_{k=1}^{t} \left(\frac{1}{\w_{k+1:T}} - \frac{1}{\w_{k:T}}\right) \w_t \right]-\w_t\left(\frac{1}{\w_{t+1:T}} - \frac{1}{\w_{t:T}}\right)\\
    &=\frac{\w_t}{\w_{1:T}} + \sum_{k=1}^{t-1} \left(\frac{1}{\w_{k+1:T}} - \frac{1}{\w_{k:T}}\right) \w_t\\
    &=\frac{\w_t}{\w_{t:T}}.
\end{align*}
Next, for $p<t$, the coefficient of $\langle g_t, x_p\rangle$ is:
\begin{align*}
    -\left(\frac{1}{\w_{p+1:T}}-\frac{1}{\w_{p:T}}\right)\w_t.
\end{align*}
And for $p>t$, the coefficient of $\langle g_t, x_p\rangle$ is zero. Finally the coefficient of $\langle g_t, \cmp\rangle$ is $-\frac{\w_t}{\w_{1:T}}$.

Putting this all together, we can rearrange the expression as follows:
\begin{align*}
    \E[q_T]&\le \E\left[ \sum_{t=1}^T \left\langle g_t, \frac{\w_t}{\w_{t:T}}x_t -\left(\frac{\w_t \cmp}{\w_{1:T}}  +\sum_{p=1}^{t-1} \w_t x_t \left(\frac{1}{\w_{p+1:T}}-\frac{1}{\w_{p:T}}\right) \right)\right\rangle \right]\\
    &= \E\left[ \sum_{t=1}^T\left\langle \frac{\w_t}{\w_{t:T}} g_t, x_t - \w_{t:T}\left( \frac{\cmp}{\w_{1:T}} + \sum_{p=1}^{t-1} x_p\left(\frac{1}{\w_{p+1:T}}-\frac{1}{\w_{p:T}}\right)\right) \right\rangle\right].
\end{align*}

Now, given an arbitrary sequence $z_1,\dots,z_T$, define $x_t$ recursively by:
\begin{align*}
    x_t = \w_{t:T}\left( \frac{z_t}{\w_{1:T}} + \sum_{p=1}^{t-1} x_p\left(\frac{1}{\w_{p+1:T}}-\frac{1}{\w_{p:T}}\right)\right).
\end{align*}
Then we have:
\begin{align*}
    \E[q_T]&\le \E\left[ \sum_{t=1}^T\left\langle \frac{\w_t}{\w_{t:T}} g_t, x_t - \w_{t:T}\left( \frac{\cmp}{\w_{1:T}} + \sum_{p=1}^{t-1} x_p\left(\frac{1}{\w_{p+1:T}}-\frac{1}{\w_{p:T}}\right)\right) \right\rangle\right]\\
    &=\E\left[\sum_{t=1}^T\frac{\w_t}{\w_{1:T}} \langle g_t, z_t - \cmp\rangle \right].
\end{align*}
\end{proof}

Now, we are finally ready to prove Theorem~\ref{thm:additive}:
\thmadditive*
\begin{proof}
Let's define $\hat x_1=z_t$ and recursively set:
\begin{align*}
    \hat x_t = w_{t:T} \left(\frac{z_t}{w_{1:T}} + \sum_{p=1}^{t-1} \hat x_p\left(\frac{1}{w_{p+1:T}} - \frac{1}{w_{p:T}}\right)\right).
\end{align*}
Then, Lemma~\ref{lem:rearrange} shows that $\E[f(\hat x_T) - f(\cmp)] \le \E\left[\frac{1}{\w_{1:T}}\sum_{t=1}^T\w_t\langle g_t, z_t - \cmp\rangle \right]$. So, it suffices to show that $x_t=\hat x_t$ for all $t$. In turn, since $\hat x_1 = z_1=x_1$, it suffices to show $\hat x_{t+1}-\hat x_t = \frac{w_{t+1:T}}{w_{1:T}} \Delta_t=x_{t+1}-x_t$ for all $t$.

To this end, let's do some calculation. First:
\begin{align*}
    \hat x_t &=   \frac{\w_t}{\w_{t:T}}\hat x_t + \frac{\w_{t+1:T}}{\w_{t:T}}\hat x_t\\
    &= \frac{\w_t}{\w_{t:T}} \hat x_t +\w_{t+t:T}\left( \frac{z_t}{\w_{1:T}} + \sum_{p=1}^{t-1} \hat x_p\left(\frac{1}{\w_{p+1:T}}-\frac{1}{\w_{p:T}}\right)\right).
\end{align*}
With this expression, we have:
\begin{align*}
\hat x_{t+1} - \hat x_t &= \hat x_{t+1} - \w_{t+1:T}\left( \frac{z_{t+1}}{\w_{1:T}} + \sum_{p=1}^{t} x_p\left(\frac{1}{\w_{p+1:T}}-\frac{1}{\w_{p:T}}\right)\right)  - \frac{\w_t}{\w_{t:T}} \hat x_t \\
&=w_{t+1:T} \left(\frac{z_{t+1}}{w_{1:T}} + \sum_{p=1}^{t} \hat x_p\left(\frac{1}{w_{p+1:T}} - \frac{1}{w_{p:T}}\right)\right)\\
&\qquad-  \w_{t+1:T}\left( \frac{z_t}{\w_{1:T}} + \sum_{p=1}^{t-1} x_p\left(\frac{1}{\w_{p+1:T}} -\frac{1}{\w_{p:T}}\right)\right) - \frac{\w_t}{\w_{t:T}} \hat x_t \\
&= \frac{\w_{t+1:T} (z_{t+1} -  z_t)}{\w_{1:T}} + \w_{t+1:T} \hat x_t\left(\frac{1}{\w_{t+1:T}} - \frac{1}{\w_{t:T}}\right) - \frac{\w_t}{\w_{t:T}} \hat x_t\\
&= \frac{\w_{t+1:T}}{\w_{1:T}} \Delta_t + \hat x_t\left(1-\frac{\w_{t+1:T} +\w_t}{\w_{t:T}}\right)\\
&= \frac{\w_{t+1:T}}{\w_{1:T}} \Delta_t.
\end{align*}
\end{proof}

\section{Proof of Theorem~\ref{thm:optimal}}\label{sec:proof:optimal}

\thmoptimal*

\begin{proof}
    First, observe that with $\Delta_t = -\w_t g_t$, we have $x_{t+1} = x_t + \frac{\w_{t+1:T}}{\w_{1:T}} \Delta_t$. Therefore with $z_1=x_1$ and $z_{t+1} = z_t +\Delta_t$, Theorem~\ref{thm:additive} implies:
    \begin{align*}
        \E[f(x_T) - f(\cmp)]&\le \E\left[\frac{1}{\w_{1:T}} \sum_{t=1}^T \langle \w_t g_t,z_t -\cmp\rangle\right].
    \end{align*}
    Next, observe that $z_t$ is simply online gradient descent with learning rate 1 acting on the loss vectors $\w_t z_t$. Standard analysis \citep{zinkevich2003online} shows:
    \begin{align*}
        \sum_{t=1}^T \langle \w_t g_t,z_t -\cmp\rangle& = \frac{\|z_1-\cmp\|^2}{2} - \frac{\|z_{T+1}-\cmp\|^2}{2} +\sum_{t=1}^T \frac{\w_t^2 \|g_t\|^2}{2}.
    \end{align*}
    This immediately implies the first part of the Theorem. Next, we need to solve for the minimizing values of $\w_t$. To do this, we take the logarithm of the expression $\frac{1}{2\w_{1:T}}\left(\|x_1-\cmp\|^2 + \sum_{t=1}^T \w_t^2\|g_t\|^2\right)$ and differentiate:
    \begin{align*}
        \frac{\partial}{\partial \w_k} \log\left[\frac{1}{2 \w_{1:T}}\left(\|x_1-\cmp\|^2 + \sum_{t=1}^T \w_t^2\|g_t\|^2\right)\right]&= \frac{2\w_k \|g_k\|^2}{\|x_1-\cmp\|^2 + \sum_{t=1}^T \w_t^2\|g_t\|^2} -\frac{1}{\w_{1:T}}.
    \end{align*}
    We set this equal to zero to solve for the optimal $\w_k$:
    \begin{align*}
    \w_k &= \|g_k\|^{-2}\frac{\|x_1-\cmp\|^2 + \sum_{t=1}^T \w_t^2\|g_t\|^2}{2 \w_{1:T}}\triangleq \lambda \|g_k\|^{-2},
    \end{align*}
    where we have defined $\lambda =\frac{\|x_1-\cmp\|^2 + \sum_{t=1}^T \w_t^2\|g_t\|^2}{2 \w_{1:T}} $, which does not depend on $k$. That is, the optimal $\w_k$ value is proportional to $\|g_k\|^{-2}$. With this expression, we have:
    \begin{align*}
        \sum_{t=1}^T \w_t^2\|g_t\|^2 &= \lambda^2 \sum_{t=1}^T \|g_t\|^{-2}\\
        \w_{1:T} &= \lambda \sum_{t=1}^T \|g_t\|^{-2}.
    \end{align*}
    So, let us now solve for $\lambda$ by plugging in these values:
    \begin{align*}
        \lambda &=\frac{\|x_1-\cmp\|^2 + \sum_{t=1}^T \w_t^2\|g_t\|^2}{2 \w_{1:T}} \\
        \lambda &= \frac{\|x_1-\cmp\|^2 + \lambda^2 \sum_{t=1}^T \|g_t\|^{-2}}{2\lambda \sum_{t=1}^T \|g_t\|^{-2}} \\
        \lambda &= \frac{\|x_1-\cmp\|^2}{2\lambda \sum_{t=1}^T \|g_t\|^{-2}} + \frac{\lambda}{2}\\
        \lambda &=\frac{\|x_1-\cmp\|}{\sqrt{\sum_{t=1}^T \|g_t\|^{-2}}}.
    \end{align*}
    This in turn implies the claimed optimal value for $\w_k$.
\end{proof}

\section{Schedules for Per-Coordinate Updates}\label{sec:percoordinate}

Many popular optimization algorithms in use today like Adam \citep{KingmaB14} and its variants employ \emph{per-coordinate} updates: the update $\Delta_t$ is not proportional to the gradient $g_t$ but instead scales each coordinate of $g_t$ by an adaptively chosen value. In this section we propose an approximately optimal schedule for such methods. 

\begin{restatable}{theorem}{thmpercoordinate}\label{thm:percoordinate}
Suppose $\Delta_t =z_{t+1}-z_t= -\w_t \cdot(\eta_t \odot g_t)$ where $\eta_t \in \R^d$ is a vector of learning rates and $\odot$ indicates coordinate-wise product. Set $x_{t+1} = x_t - \frac{\w_{t+1:T}}{\w_{1:T}} \Delta_t$. Define the quantity $R$ by:
\begin{align*}
    R= \sqrt{\sum_{t=1}^T \sum_{i=1}^d \frac{(z_{t,i}-{\cmp}_i)^2 - (z_{t+1,i}-{\cmp}_i)^2}{\eta_{t,i}}}
\end{align*}
Then we have:
\begin{align*}
    \E[f(x_T) - f(\cmp)]&\le \E\left[\frac{1}{\w_{1:T}} \left(\frac{R^2}{2} + \sum_{t=1}^T \w_t^2 \sum_{i=1}^d \eta_{t,i} g_{t,i}^2\right)\right].
\end{align*}
Moreover, for any given fixed values for $R$ and $g_{t,i}^2$, the expression $\frac{1}{\w_{1:T}} \left(\frac{R^2}{2}+ \sum_{t=1}^T \w_t^2 \sum_{i=1}^d \eta_{t,i}^2 g_{t,i}^2\right)$ is minimized by setting $w_t$ as below:
\begin{align*}
    \w_t &= \frac{R}{\sqrt{ \sum_{t=1}^T \left( \sum_{i=1}^d \eta_{t,i} g_{t,i}^2\right)^{-1}}} \cdot \left(\sum_{i=1}^d \eta_{t,i} g_{t,i}^2\right)^{-1}.
\end{align*}
\end{restatable}
As an example of how to use this result, let us suppose we are employing the Adam optimizer, and let us also also ignore the presence of momentum when computing the weights (so really, these will be the weights for the RMSProp optimizer). In this case, $\eta_{t,i} \propto \frac{1}{\sqrt{v_{t,i}}}$ where $v_{t,i}$ is the exponential average of the squared gradients. Thus, we get:
\begin{align*}
    w_t = \lambda \left(\sum_{i=1}^d \frac{g_{t,i}^2}{\sqrt{v_{t,i}}}\right)^{-1}
\end{align*}
for some $\lambda$. Then, we use a learning rate schedule of $\frac{w_t w_{t+1:T}}{w_{1:T}}$.

That is, the corresponding procedure to Algorithm~\ref{alg:refinement} is given by Algorithm~\ref{alg:refinement_adam}:

\begin{algorithm}[ht]
\begin{algorithmic}[1]
    \STATE {\bfseries Input:} $G=\left(G_t = \sum_{i=1}^d \frac{g_{t,i}^2}{\sqrt{v_{t,i}}}\right)$ length $T$ sequence of weighted gradient norms from Adam optimizer, smoothing hyper-parameter $\tau > 0$
    \STATE $\hat{G} = \textrm{median\_filter}(G, \textrm{filter\_width}=\tau\,T, \textrm{padding}=(\textrm{nearest}, \textrm{reflect}))$
    \STATE Define $w_{t}=\hat{G}_{t}^{-1}$
    \STATE For each $t$, let:
        \STATE \[
    \eta_{t} = w_{t} \sum_{p=t+1}^{T}w_{p}
    \]
    \STATE Return normalized schedule $\eta/\max(\eta)$
\end{algorithmic}
\caption{\label{alg:refinement_adam}Schedule Refinement for Adam}
\end{algorithm}

In practice, however, recording the value $\w_t\propto \left(\sum_{i=}^d \frac{g_{t,i}^2}{\sqrt{v_{t,i}}}\right)^{-1}$ may be difficult (for example, in Pytorch, the Adam implementation is primarily C code rather than Python, which makes it substantially more involved to modify). However, by inspecting the formula for $\w_t$, we can see that it is likely to be an interpolation between the $1/\|g_t\|_2^2$ and $1/\|g_t\|_1$ (the L1 norm is not squared). The intuition behind this is that if $v_{t,i}$ has a minimum value of $|g_{t,i}|$. With this minimum value, $w_t\propto 1/\|g_t\|_1$. On the other hand if all $v_{t,i}$ are the same constant, then $w_t\propto 1/\|g_t\|_2^2$. In practice we expect behavior closer to the first case and recommend using $\w_t\propto 1/\|g_t\|_1$.

\begin{proof}[proof of Theorem~\ref{thm:percoordinate}]
By standard online gradient descent analysis, we have:
\begin{align*}
     \langle \w_t \cdot g_t, z_t - \cmp\rangle & =  \sum_{i=1}^d \w_tg_{t,i}(z_{t,i} - {\cmp}_i)\\
     &= \sum_{i=1}^d \frac{(z_{t,i}-{\cmp}_i)^2}{2\eta_{t,i}} - \frac{(z_{t+1,i}-{\cmp}_i)^2}{2\eta_{t,i}} + \frac{\eta_{t,i}}{2} \w_t^2 g_{t,i}^2.
\end{align*}
Summing this over $t$ from $1$ to $T$, we get
\begin{align*}
     \sum_{t=1}^T \langle \w_t \cdot g_t, z_t - \cmp\rangle &= \sum_{t=1}^T \sum_{i=1}^d \frac{(z_{t,i}-{\cmp}_i)^2}{2\eta_{t,i}} - \frac{(z_{t+1,i}-{\cmp}_i)^2}{2\eta_{t,i}} + \sum_{t=1}^T \w_t^2\sum_{i=1}^d \frac{\eta_{t,i} g_{t,i}^2}{2}\\
     &=\frac{R^2}{2} + \sum_{t=1}^T \frac{\w_t^2}{2}\sum_{i=1}^d \eta_{t,i} g_{t,i}^2.
\end{align*}
The first part of the result now follows from Theorem~\ref{thm:additive}.

Next, we again take the logarithm, differentiate and set equal to zero:
\begin{align*}
    0&=\frac{\partial}{\partial \w_k} \log\left(\frac{1}{\w_{1:T}} \left(\frac{R^2}{2} + \sum_{t=1}^T \frac{\w_t^2}{2} \sum_{i=1}^d \eta_{t,i} g_{t,i}^2\right)\right)\\
    &= \frac{2\w_k \sum_{i=1}^d \eta_{k,i}^2 g_{k,i}^2}{R^2 + \sum_{t=1}^T \w_t^2 \sum_{i=1}^d \eta_{t,i} g_{t,i}^2} - \frac{1}{\w_{1:T}}.
\end{align*}
Rearranging, we obtain
\begin{align*}
    \w_k &= \frac{\left(\sum_{i=1}^d \eta_{k,i}^2 g_{k,i}^2\right)^{-1}\left(R^2 + \sum_{t=1}^T \w_t^2 \sum_{i=1}^d \eta_{t,i} g_{t,i}^2\right)}{2w_{1:T}}\\
    &\triangleq  \lambda \left(\sum_{i=1}^d \eta_{k,i} g_{k,i}^2\right)^{-1},
\end{align*}
where we have collected the non $k$-dependent terms into $\lambda=\frac{R^2 + \sum_{t=1}^T \w_t^2 \sum_{i=1}^d \eta_{t,i} g_{t,i}^2}{2w_{1:T}} $. Now we solve for $\lambda$:
\begin{align*}
    \lambda &= \frac{R^2 + \sum_{t=1}^T \w_t^2 \sum_{i=1}^d \eta_{t,i} g_{t,i}^2}{2w_{1:T}}\\
    &=\frac{R^2 + \lambda^2 \sum_{t=1}^T\left( \sum_{i=1}^d \eta_{t,i} g_{t,i}^2\right)^{-1}}{2\lambda \sum_{t=1}^T \left( \sum_{i=1}^d \eta_{t,i} g_{t,i}^2\right)^{-1}}\\
     &= \frac{R^2}{2\lambda \sum_{t=1}^T \left( \sum_{i=1}^d \eta_{t,i} g_{t,i}^2\right)^{-1}} + \frac{\lambda}{2}\\
     &= \frac{R}{\sqrt{ \sum_{t=1}^T \left( \sum_{i=1}^d \eta_{t,i} g_{t,i}^2\right)^{-1}}}.
\end{align*}
Putting $w_t=\lambda \left(\sum_{i=1}^d \eta_{t,i} g_{t,i}^2\right)^{-1}$ completes the proof.
\end{proof}

\section{Explicit Bounds for Arbitrary Schedules}
\label{sec:arbitary}

Section~\ref{sec:main-result} develops a framework that suggests using learning rates of the form $\eta_t =  \frac{\w_t \w_{t+1:T}}{\w_{1:T}}$ and in Theorem~\ref{thm:optimal} we provided a simple closed-form solution for the optimal values of $\w_t$. Given this apparently restricted form  of $\eta_t$, it is natural to ask if this restriction is real: that is, can \emph{every} schedule  $\eta_1,\dots,\eta_{T-1}$ be represented using some weights $\w_t$? Theorem~\ref{thm:represent} shows that the answer is ``yes'':

\begin{restatable}{theorem}{thmrepresent}\label{thm:represent}
Let $\eta_1,\dots,\eta_{T-1}$ be a sequence of non-negative numbers with $\eta_1\ge 0$. Then there is a sequence of non-negative weights $\w_1,\dots,\w_T$ such that $\eta_t =  \frac{\w_t \w_{t+1:T}}{\w_{1:T}}$ for all $t$.
\end{restatable}

\begin{proof}
First, observe that for any solution to the desired identity $\eta_t = \frac{\w_t \w_{t:1:T}}{\w_{1:T}}$,  replacing $w$ with $c\cdot  w$ for some constant $c$ will yield a solution for $\eta$ replaced with $c\cdot \eta$. Thus, it suffices to consider the possibility that $\max_t \eta_t =1$.

The intuition for the rest of the proof is the following: Given the value for $\w_1$, we can solve for $\w_{2:T}$ using the equation $\frac{\w_t \w_{2:T}}{\w_1 + \w_{2:T}} = \eta_1$. This in turn provides us with the value of $\w_{1:T}$. Now, given $\w_{k:T}$ for any $k$ we can solve for $\w_{k}$ using the equation $\eta_k = \frac{\w_k \w_{k+1:T}}{\w_{1:T}} = \frac{\w_k (\w_{k:T} - \w_k)}{\w_{1:T}}$. Thus we may recursively compute all  the values of $w$. Each of these steps requires solving a quadratic equation.  We simply choose an initial $\w_1$ so as to ensure that all the following quadratic equations have non-negative real roots.

Specifically, set $\w_1 = \frac{2^{2T} + \sqrt{2^{4T} - 4\eta_1}}{2}$ and define $s_1=2^{2T}$. Then, recursively define for $t=2,\dots,T-1$:
\begin{align*}
    s_{t} &= s_{t-1} - \w_{t-1}\\
    \w_t &= \frac{s_t - \sqrt{s_t^2 - 4 s_1 \eta_t}}{2}
\end{align*}
and set $\w_T=s_T= s_{T-1}-\w_{T-1}$. Notice that these choices satisfy:
\begin{align*}
    \w_t^2 - s_t\w_t +s_1 \eta_t=0
\end{align*}
so that if we could establish (1) that all $\w_t\ge 0$  and (2) $s_t = \w_{t:T}$, then we would have:
\begin{align*}
    \frac{\w_t \w_{t+1:T}}{\w_{1:T}} &= \frac{\w_t (\w_{t:T} - \w_t)}{\w_{1:T}}\\
    &=\frac{\w_t (s_t -\w_t)}{s_1}\\
    &=\frac{s_t \w_t-\w_t^2}{s_1}\\
    &=\eta_t
\end{align*}
as desired. 

Let us first prove (1): all $\w_t\ge 0$. For $t\le T-1$, this will  hold if $s_t^2 - 4s_1\eta_t>0$ . To establish this, we will first show that $s_t \ge \frac{s_1}{2^{t-1}}$ for all $t\le T-1$. If this holds, then we have:
\begin{align*}
    s_t^2 \ge \frac{s_1^2}{2^{2t-2}}\ge \frac{s_1 2^{2T}}{2^{2t-2}}\ge 4s_1\ge 4s_1\eta_t
\end{align*}
for $t\le T-1$, where we have used our assumption $\eta_t\le 1$.

So, we now establish $s_t \ge \frac{s_1}{2^{t-1}}$ by induction for $t\le T-1$. The statement is clear for $t=1$. Suppose it holds for for all $t\le k$ for some $k\le T-2$. Then we have:
\begin{align*}
    s_{k+1} &= s_k - \w_k\\
    &= \frac{s_k + \sqrt{s_k^2 - 4s_1\eta_k}}{2}\\
    &\ge \frac{s_k}{2}\\
    &\ge \frac{s_1}{2^{k-1+1}},
\end{align*}
which establishes the claim. Therefore for $t\le T-1$, $\w_t$ are non-negative real numbers. Finally, we have:
\begin{align*}
    \w_{T-1}  &= \frac{s_{T-1}+ \sqrt{s_{T-1}^2 - 4 s_1 \eta_t}}{2}\le s_{T-1},
\end{align*}
so that $\w_T=s_T = s_{T-1}-\w_{T-1}\ge 0$. Thus $\w_t\ge 0$ for all $t$.

Next, to show (2): $s_t = \w_{t:T}$. This is nearly immediate. By definition we have $\w_T=s_T$. Suppose $s_k = \w_{t:T}$ for some $t\ge 2$. Then $s_t= s_{t-1}-\w_{t-1}$ so that $s_{t-1}=s_t+\w_{t-1}=\w_{t-1:T}$. Thus by induction we have $s_t = \w_{t:T}$ for all $t$, which is the last thing we needed to show.
\end{proof}
This Theorem shows that by optimizing the weights $w_t$, we are in some sense also solving for the optimal learning rate $\eta_t$. However, notice that the reverse is not true: any given schedule $\eta_t$ can be represented by a number of different weights $\w_t$, and these weights give rise to different bounds using Theorem~\ref{thm:additive}. The proof of Theorem~\ref{thm:represent} works by constructing a particular set of weights $\w_t$, but these may not be the best weights in terms of providing the tightest convergence bound for the given $\eta_t$.

Although Theorem~\ref{thm:represent} shows that the learning rate representation of Theorem~\ref{thm:optimal} indeed covers all possible schedules, it does not provide a user-friendly way to analyze the convergence of an arbitrary schedule. In this section, we provide an alternative analysis that fills this gap.

This approach is closer to previous final-iterate analysis techniques: we bound the final iterate in terms of the average iterate, plus an additional \emph{error} term. This bound is strictly looser than those of Theorems~\ref{thm:additive} and \ref{thm:optimal} in the constant factors, but provides a convenient way to analyze arbitrary schedules. The bound is presented in Theorem~\ref{thm:anyetabound} below.
\begin{restatable}{theorem}{thmanyetabound}\label{thm:anyetabound}
Suppose that $f$ is convex and let $x_t$ be given by SGD with learning rates $\eta_t$: $x_{t+1}= x_t - \eta_t g_t$. Then for the last iterate $x_T$ we have:
\begin{align}
\E[f(x_{T})-f(\cmp)] & \leq\E\left[\frac{1}{2\sum_{t=1}^{T}\eta_{t}}D^{2}+\frac{1}{2\sum_{t=1}^{T}\eta_{t}}\sum_{t=1}^{T}\eta_{t}^{2}\left\Vert g_{t}\right\Vert ^{2}\right]\nonumber\\
 & +\E\left[\frac{1}{2}\sum_{k=1}^{T-1}\frac{\eta_{k}}{\sum_{t=k+1}^{T}\eta_{t}}\left(\frac{1}{\sum_{t=k}^{T}\eta_{t}}\sum_{t=k}^{T}\eta_{t}^{2}\left\Vert g_{t}\right\Vert ^{2}\right)\right]. \label{eq:key-bound}
\end{align}
\end{restatable}
Optimizing this bound with respect to the step-size sequence produces schedules that are visually indistinguishable from those of the regret based approach described in Section~\ref{sec:optimize} for large $T$.

To prove Theorem~\ref{thm:anyetabound}, we will need to control quantities like:
\[
\sum_{t=k}^{T}\eta_{t}\left(q_{t}-q_{k}\right)=f(x_{t})-f(x_{k}).
\]
This can be bounded by the usual suboptimality inequality for SGD/GD methods,
which holds for any $\cmp$:
\[
\sum_{t=k}^{T}\eta_{t}\left[f(x_{t})-f(\cmp)\right]\leq\frac{1}{2}\left\Vert x_{k}-\cmp\right\Vert ^{2}+\frac{1}{2}\sum_{t=k}^{T}\eta_{t}^{2}\left\Vert g_{t}\right\Vert ^{2},
\]
Setting $\cmp=x_{k}$ yields:
\[
\sum_{t=k}^{T}\eta_{t}\left[q_{t}-q_{k}\right]\leq\frac{1}{2}\sum_{t=k}^{T}\eta_{t}^{2}\left\Vert g_{t}\right\Vert ^{2}.
\]
We can use this in our Lemma~\ref{lem:tail} to obtain the following result:
\begin{corollary}\label{cor:anyeta}
\[
q_{T}=\frac{1}{\sum_{t=1}^{T}\eta_{t}}\sum_{t=1}^{T}\eta_{t}q_{t}+\frac{1}{2}\sum_{k=1}^{T-1}\frac{\eta_{k}}{\sum_{t=k+1}^{T}\eta_{t}}\left(\frac{1}{\sum_{t=k}^{T}\eta_{t}}\sum_{t=k}^{T}\eta_{t}^{2}\left\Vert g_{t}\right\Vert ^{2}\right).
\]
\end{corollary}
\subsection{Proof of Theorem~\ref{thm:anyetabound}}
\begin{proof}
Following the standard convergence bound approach:
\begin{align*}
\left\Vert x_{t+1}-\cmp\right\Vert ^{2} & =\left\Vert x_{t}-\eta_{t}g_{t}-\cmp\right\Vert ^{2}\\
 & =\left\Vert x_{t}-\cmp\right\Vert ^{2}-2\eta_{t}\left\langle g_{t},x_{t}-\cmp\right\rangle +\eta_{t}^{2}\left\Vert g_{t}\right\Vert ^{2}\\
 & \leq\left\Vert x_{t}-\cmp\right\Vert ^{2}-2\eta_{t}\left[f(x_{t})-f(\cmp)\right]+\eta_{t}^{2}\left\Vert g_{t}\right\Vert ^{2}.
\end{align*}
Summing over $t$ and telescoping gives:
\[
\sum_{t=1}^{T}\eta_{t}\left[f(x_{t})-f(\cmp)\right]\leq\frac{1}{2}D^{2}+\sum_{t=1}^{T}\eta_{t}^{2}\left\Vert g_{t}\right\Vert ^{2}.
\]
Then divide through by $\sum_{t=1}^{T}\eta_{t}$:
\[
\frac{1}{\sum_{t=1}^{T}\eta_{t}}\sum_{t=1}^{T}\eta_{t}\left[f(x_{t})-f_{*}\right]\leq\frac{1}{2\sum_{t=1}^{T}\eta_{t}}D^{2}+\frac{1}{2\sum_{t=1}^{T}\eta_{t}}\sum_{t=1}^{T}\eta_{t}^{2}\left\Vert g_{t}\right\Vert ^{2}.
\]
Now we apply Corollary~\ref{cor:anyeta} to get:
\begin{align*}
f(x_{T})-f_{*} & \leq\frac{1}{2\sum_{t=1}^{T}\eta_{t}}D^{2}+\frac{1}{2\sum_{t=1}^{T}\eta_{t}}\sum_{t=1}^{T}\eta_{t}^{2}\left\Vert g_{t}\right\Vert ^{2}\\
 & +\frac{1}{2}\sum_{k=1}^{T-1}\frac{\eta_{k}}{\sum_{t=k+1}^{T}\eta_{t}}\left(\frac{1}{\sum_{t=k}^{T}\eta_{t}}\sum_{t=k}^{T}\eta_{t}^{2}\left\Vert g_{t}\right\Vert ^{2}\right).
\end{align*}
\end{proof}

By using a worst-case bound of $\|g_t\|^2\le G^2$, we obtain:

\begin{corollary}
If $f$ is $G$-Lipschitz, then
\begin{align*}
f(x_{T})-f_{*} & \leq\frac{1}{2\sum_{t=1}^{T}\eta_{t}}D^{2}+\frac{G^{2}}{2\sum_{t=1}^{T}\eta_{t}}\sum_{t=1}^{T}\eta_{t}^{2}\\
 &\quad +\frac{G^{2}}{2}\sum_{k=1}^{T-1}\frac{\eta_{k}}{\sum_{t=k+1}^{T}\eta_{t}}\left(\frac{1}{\sum_{t=k}^{T}\eta_{t}}\sum_{t=k}^{T}\eta_{t}^{2}\right).
\end{align*}
\end{corollary}

\subsection{Linear schedule analysis with Theorem~\ref{thm:anyetabound}}
In this section, we use Theorem~\ref{thm:anyetabound} to re-analyze the the linear decay schedule:
\begin{equation}
\eta_{t}=\frac{D}{G\sqrt{T}}\left(1-\frac{t}{T+1}\right).\label{eq:lindec-app}
\end{equation}
The resulting convergence rate is asymptotically correct, but does not achieve the optimal constant obtained by Theorem~\ref{thm:optimal}.
\begin{theorem}
Schedule \ref{eq:lindec-app} gives the following bound on the last iterate:
\[
f(x_{T})-f_{*}\leq\left(2+\frac{1}{4}\right)\frac{DG}{\sqrt{T}},
\]
or more precisely:
\[
f(x_{T})-f_{*}\leq\left(2+\frac{H\left(T-1\right)-2/3}{T+1}\right)\frac{DG}{\sqrt{T}},
\]
where H(T) is the Tth harmonic sum.
\end{theorem}

\begin{proof}
We start with the above bound:

\begin{align*}
f(x_{T})-f_{*} & \leq\frac{1}{2\sum_{t=1}^{T}\eta_{t}}D^{2}+\frac{G^{2}}{2\sum_{t=1}^{T}\eta_{t}}\sum_{t=1}^{T}\eta_{t}^{2}\\
 & +\frac{G^{2}}{2}\sum_{k=1}^{T-1}\frac{\eta_{k}}{\sum_{t=k+1}^{T}\eta_{t}}\left(\frac{1}{\sum_{t=k}^{T}\eta_{t}}\sum_{t=k}^{T}\eta_{t}^{2}\right).
\end{align*}
Since the $\frac{D}{G\sqrt{T}}$ part of the step size is constant,
it can be pulled outside the summations, so we just need to focus
on summations involving $\left(1-\frac{t}{T+1}\right)$. For the first
two terms we have in the bound, they simplify

\begin{align*}
\sum_{t=1}^{T}\eta_{t} & \propto\sum_{t=1}^{T}\left(1-\frac{t}{T+1}\right)\\
 & =\left(T-\frac{1}{T+1}\sum_{t=1}^{T}t\right)\\
 & =\left(T-\frac{1}{2}\frac{T(T+1)}{T+1}\right)\\
 & =\frac{T}{2}.
\end{align*}
Similarly,
\begin{align*}
\sum_{t=1}^{T}\eta_{t}^{2} & \propto\sum_{t=1}^{T}\left(1-\frac{t}{T+1}\right)^{2}\\
 & =\sum_{t=1}^{T}\left(1-\frac{2t}{T+1}+\frac{t^{2}}{T+1^{2}}\right)\\
 & =\left(T-T+\sum_{t=1}^{T}\frac{t^{2}}{T^{2}}\right)\\
 & =\left(\frac{T(T+1)(2T+1)}{6(T+1)^{2}}\right)\\
 & =\left(\frac{T(T+1)(T+1/2)}{3(T+1)^{2}}\right)\\
 & \leq\frac{T}{3}.
\end{align*}
So we have the following bound on the last iterate:

\begin{align*}
f(x_{T})-f_{*} & \leq\frac{DG}{\sqrt{T}}+\frac{DG}{3\sqrt{T}}\\
 & +\frac{G^{2}}{2}\sum_{k=1}^{T-1}\frac{\eta_{k}}{\sum_{t=k+1}^{T}\eta_{t}}\left(\frac{1}{\sum_{t=k}^{T}\eta_{t}}\sum_{t=k+1}^{T}\eta_{t}^{2}\right).
\end{align*}

To simplify the remaining term we rely on computer algebra software.
SymPy gives:
\[
\sum_{k=1}^{T-1}\frac{\eta_{k}}{\sum_{t=k+1}^{T}\eta_{t}}\left(\frac{1}{\sum_{t=k}^{T}\eta_{t}}\sum_{t=k+1}^{T}\eta_{t}^{2}\right)=\left(\frac{D}{G\sqrt{T}}\right)\frac{4T+6H\left(T-1\right)-4}{3(T+1)}.
\]

Where $H(T)=\sum_{t=1}^{T}1/t$ is the harmonic sum. So:
\begin{align*}
f(x_{T})-f_{*} & \leq\frac{DG}{\sqrt{T}}+\frac{DG}{3\sqrt{T}}+\frac{2DG\sqrt{T}}{3T}+\frac{DG\left(3H\left(T-1\right)-2\right)}{3\left(T+1\right)\sqrt{T}}\\
 & \leq\frac{2DG}{\sqrt{T}}+\frac{DG}{3\sqrt{T}}+\frac{DG\left(3H\left(T-1\right)-2\right)}{3(T+1)\sqrt{T}}.
\end{align*}
Note that the term $\left(3H\left(T-1\right)-2\right)/3\left(T+1\right)\leq1/4$
for all $T$, so:
\begin{align*}
\frac{DG\left(3H\left(T-1\right)-2\right)}{3(T+1)\sqrt{T}} & =\frac{DG}{4\sqrt{T}},
\end{align*}
So combining with the $\frac{DG}{\sqrt{T}}+\frac{DG}{3\sqrt{T}}$
terms, we have:

\[
f(x_{T})-f_{*}\leq\left(2+\frac{1}{4}\right)\frac{DG}{\sqrt{T}}.
\]
bounding the harmonic function with a log gives instead:
\[
f(x_{T})-f_{*}\leq\frac{2DG}{\sqrt{T}}+O\left(\frac{DG\log(T)}{T^{3/2}}\right).
\]
\end{proof}

\section{Experimental Setup}\label{sec:experiment_details}
Our experiments on CIFAR-10, CIFAR-100, ImageNet and RCNN use SGD, and the remaining problems use Adam. We used decoupled weight decay with Adam in each case following standard practice for each problem.

\subsection{Convex experiments}
Each dataset is obtained from the LIBSVM repository with used without modifications. 

\parbox{0.5\textwidth}{
\raggedleft
\begin{tabular}{|c|c|}
\hline
\textbf{Hyper-parameter}  & \textbf{Value}\tabularnewline
\hline 
GPUs  & 1$\times $V100\tabularnewline
\hline 
Batch size & 16\tabularnewline
\hline 
Epochs & 100\tabularnewline
\hline 
Seeds & 10\tabularnewline
\hline 
\end{tabular}
}\quad
\parbox{0.5\textwidth}{
\raggedright
\begin{tabular}{|c|c|}
\hline
\textbf{Hyper-parameter}  & \textbf{Value}\tabularnewline
\hline 
Decay & 0.0\tabularnewline
\hline 
Optimizer & Adam \tabularnewline
\hline 
$\beta_1$ & 0.9 \tabularnewline
\hline 
$\beta_2$ & 0.95 \tabularnewline
\hline 
\end{tabular}
}

\begin{figure}\label{fig:convex-schedules}
\center
\includegraphics[width=\textwidth]{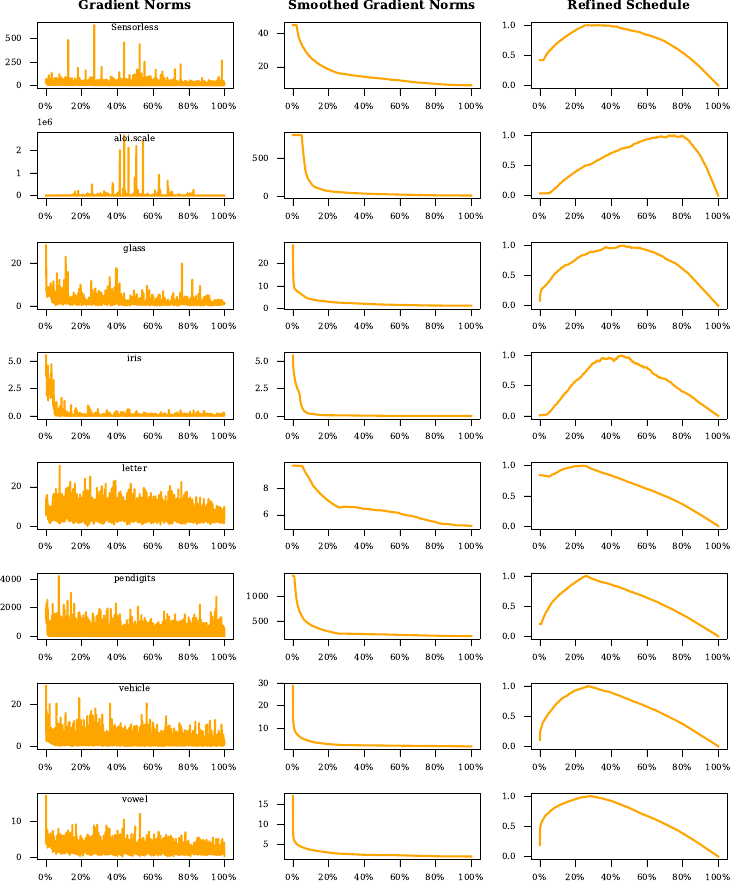}
\caption{\label{fig:convex-chedules}Logistic regression schedules, generated using a linear decay schedule with warmup for the initial run.}
\end{figure}

\subsection{CIFAR-10}
We used custom training code based on the PyTorch tutorial code for this problem. Following standard data-augmentation practises, we appliyed random horizontal flips and random offset cropping down to 32x32, using reflection padding of 4 pixels. Input pixel dataD was normalized by centering around 0.5.

\begin{tabular}[t]{|c|c|}
\hline 
\textbf{Hyper-parameter}  & \textbf{Value}\tabularnewline
\hline
Architecture  & Wide ResNet 16-8\tabularnewline
\hline 
Epochs  & 300\tabularnewline
\hline 
GPUs  & 1$\times $V100\tabularnewline
\hline 
Batch size per GPU  & 128\tabularnewline
\hline 
\end{tabular}
\quad
\begin{tabular}[t]{|c|c|}
\hline 
\textbf{Hyper-parameter}  & \textbf{Value}\tabularnewline
\hline 
Seeds & 10\tabularnewline
\hline 
decay & 0.0001\tabularnewline
\hline 
Momentum & 0.9\tabularnewline
\hline 
\end{tabular}

\subsection{CIFAR-100}
We used the same codebase as for our CIFAR-10 experiments, with the same data augmentation. 

We normalized each input image using fixed mean and standard error values derived from pre-processing the data. 

\begin{tabular}[t]{|c|c|}
\hline 
\textbf{Hyper-parameter}  & \textbf{Value}\tabularnewline
\hline
Architecture  & \begin{tabular}{@{}c@{}} DenseNet [6,12,24,16],\\ growth rate 12\end{tabular}\tabularnewline
\hline 
Epochs  & 300\tabularnewline
\hline 
GPUs  & 1$\times $V100\tabularnewline
\hline 
\end{tabular}
\quad
\begin{tabular}[t]{|c|c|}
\hline 
\textbf{Hyper-parameter}  & \textbf{Value}\tabularnewline
\hline 
Batch size per GPU  & 64\tabularnewline
\hline 
Seeds & 10\tabularnewline
\hline 
Decay & 0.0002\tabularnewline
\hline 
Momentum & 0.9\tabularnewline
\hline 
\end{tabular}

\subsection{ImageNet}
We used the same code-base as for our CIFAR-10 experiments, and applied the same preprocessing procedure. The data-augmentations consisted of PyTorch's RandomResizedCrop, cropping to 224x224 followed by random horizontal flips.
Test images used a fixed resize to 256x256 followed by a center crop to 224x224.

\begin{tabular}[t]{|c|c|}
\hline 
\textbf{Hyper-parameter}  & \textbf{Value}\tabularnewline
\hline
Architecture  & ResNet50\tabularnewline
\hline 
Epochs  & 100\tabularnewline
\hline 
GPUs  & 8$\times $V100\tabularnewline
\hline 
Batch size per GPU  & 32\tabularnewline
\hline 
\end{tabular}
\quad
\begin{tabular}[t]{|c|c|}
\hline
\textbf{Hyper-parameter}  & \textbf{Value}\tabularnewline
\hline  
Seeds & 5\tabularnewline
\hline 
Decay & 0.0001\tabularnewline
\hline 
Momentum & 0.9\tabularnewline
\hline 
\end{tabular}

\subsection{IWSLT14}
We used the FairSeq framework \footnote{\url{https://github.com/facebookresearch/fairseq}} for our experiments here, as well as for our GPT and RoBERTa experiments. Rather than a vanilla LSTM we use the variant from \citep{wiseman-rush-2016-sequence} provided in the FairSeq codebase.

\begin{tabular}[t]{|c|c|}
\hline 
\textbf{Hyper-parameter} & \textbf{Value}\tabularnewline
\hline
Architecture  & lstm\_wiseman\_iwslt\_de\_en\tabularnewline
\hline 
Max Epoch  & 55\tabularnewline
\hline 
GPUs  & 1$\times $V100\tabularnewline
\hline 
Tokens per batch  & 4096\tabularnewline
\hline 
Warmup steps  & 4000\tabularnewline
\hline 
Dropout  & 0.3\tabularnewline
\hline 
Label smoothing  & 0.1\tabularnewline
\hline 
\end{tabular}
\quad
\begin{tabular}[t]{|c|c|}
\hline 
\textbf{Hyper-parameter} & \textbf{Value}\tabularnewline
\hline
\begin{tabular}{@{}c@{}}Share decoder, input, \\ output embed\end{tabular}  & True\tabularnewline
\hline 
Float16  & True\tabularnewline
\hline 
Update Frequency  & 1\tabularnewline
\hline 
Seeds & 10\tabularnewline
\hline 
Decay & 0.05\tabularnewline
\hline
$\beta_1, \beta_2$ & 0.9, 0.98 \tabularnewline
\hline 
\end{tabular}

\subsection{RoBERTa}
The RoBERTa implementation in FairSeq is the canonical one. We differ from the paper's results by training for a shorter duration, which is necessary to keep our experiments computationally tractable. Our BookWiki dataset matches the original paper.

\begin{tabular}[t]{|c|c|}
\hline 
\textbf{Hyper-parameter}  & \textbf{Value}\tabularnewline
\hline
Architecture  & roberta\_base\tabularnewline
\hline 
Task  & masked\_lm\tabularnewline
\hline 
Max updates  & 23,000\tabularnewline
\hline 
GPUs  & 8$\times $V100\tabularnewline
\hline 
Max tokens per sample  & 512\tabularnewline
\hline 
Dropout & 0.1\tabularnewline
\hline 
Attention Dropout & 0.1\tabularnewline
\hline 
Max sentences & 16\tabularnewline
\hline 
\end{tabular}
\quad
\begin{tabular}[t]{|c|c|}
\hline 
\textbf{Hyper-parameter}  & \textbf{Value}\tabularnewline
\hline
Warmup  & 10,000\tabularnewline
\hline 
Sample Break Mode  & Complete\tabularnewline
\hline 
Float16  & True\tabularnewline
\hline 
Update Frequency  & 16\tabularnewline
\hline 
Seeds & 5\tabularnewline
\hline 
Decay & 0.0\tabularnewline
\hline
$\beta_1, \beta_2$ & 0.9, 0.98 \tabularnewline
\hline 
\end{tabular}

\subsection{GPT}
Since the training dataset for GPT models are not availiable, we use the BookWiki dataset as used for RoBERTa training. Our model here is small, using 12 decoding layers and a decoder embedding dim of 768, giving 162 million parameters.

\begin{tabular}[t]{|c|c|}
\hline 
\textbf{Hyper-parameter}  & \textbf{Value}\tabularnewline
\hline
Architecture  & transformer\_lm\_gpt\tabularnewline
\hline 
Task  &  language\_modeling\tabularnewline
\hline 
Max updates  & 65,000\tabularnewline
\hline 
GPUs  & 8$\times $V100\tabularnewline
\hline 
Tokens per sample  & 512\tabularnewline
\hline 
Dropout & 0.1\tabularnewline
\hline 
Attention Dropout & 0.1\tabularnewline
\hline 
Max sentences & 1\tabularnewline
\hline 
Warmup  & 10,000\tabularnewline
\hline 
\end{tabular}
\quad
\begin{tabular}[t]{|c|c|}
\hline 
\textbf{Hyper-parameter}  & \textbf{Value}\tabularnewline
\hline
Sample Break Mode  & Complete\tabularnewline
\hline 
\begin{tabular}{@{}c@{}}Share decoder, input, \\ output embed\end{tabular}  & True\tabularnewline
\hline 
Float16  & True\tabularnewline
\hline 
Update Frequency  & 16\tabularnewline
\hline 
Seeds & 5\tabularnewline
\hline 
Decay & 0.005\tabularnewline
\hline
$\beta_1, \beta_2$ & 0.9, 0.98 \tabularnewline
\hline 
\end{tabular}

\subsection{ViT}
Our implementation uses the PyTorch Image Models library \footnote{\url{https://github.com/rwightman/pytorch-image-models}}, with hyper-parameters following examples given in the repository. 

\begin{tabular}[t]{|c|c|}
\hline 
\textbf{Hyper-parameter}  & \textbf{Value}\tabularnewline
\hline
Model  & vit\_tiny\_patch16\_224\tabularnewline
\hline 
GPUs  & 8$\times $V100\tabularnewline
\hline 
Epochs & 300\tabularnewline
\hline
Batch Size & 512\tabularnewline
\hline 
Warmup epochs & 5\tabularnewline
\hline 
Hflip & 0.5\tabularnewline
\hline 
aa & rand-m6-mstd0.5\tabularnewline
\hline 
mixup & 0.1\tabularnewline
\hline 
\end{tabular}
\quad
\begin{tabular}[t]{|c|c|}
\hline 
\textbf{Hyper-parameter}  & \textbf{Value}\tabularnewline
\hline
mixup & 0.1\tabularnewline
\hline 
cutmix & 1.0\tabularnewline
\hline 
Crop Pct & 0.9\tabularnewline
\hline 
BCE Loss & True\tabularnewline
\hline 
Seeds & 5\tabularnewline
\hline 
Decay & 0.1\tabularnewline
\hline
$\beta_1, \beta_2$ & 0.9, 0.999 \tabularnewline
\hline 
\end{tabular}

\subsection{DLRM}
We used a custom implementation of the DLRM model based on the publicly available code. Our optimizer  uses dense gradients for implementation simplicity, although sparse-gradients using AdaGrad is a more common baseline on this problem, we consider AdaGrad variants of our scheduling approach as future work.

\begin{tabular}[t]{|c|c|}
\hline 
\textbf{Hyper-parameter} & \textbf{Value}\tabularnewline
\hline
Iterations & 300 000\tabularnewline
\hline
Batch Size & 128\tabularnewline
\hline 
Emb Dimension & 16\tabularnewline
\hline 
GPUs  & 8$\times $V100\tabularnewline
\hline 
\end{tabular}
\quad
\begin{tabular}[t]{|c|c|}
\hline 
\textbf{Hyper-parameter} & \textbf{Value}\tabularnewline
\hline
Seeds & 5\tabularnewline
\hline 
Decay & 0.0\tabularnewline
\hline
$\beta_1, \beta_2$ & 0.9, 0.999 \tabularnewline
\hline 
\end{tabular}

\subsection{MRI}
We used the version of the the fastMRI code base at \url{https://github.com/facebookresearch/fastMRI/tree/main/banding_removal}. Note that we found that training failed using PyTorch 2 or newer, and so we ran these experiments using PyTorch 1.9.

\begin{tabular}[t]{|c|c|}
\hline 
\textbf{Hyper-parameter}  & \textbf{Value}\tabularnewline
\hline
Architecture  & 12 layer VarNet 2.0\tabularnewline
\hline 
Epochs  & 50\tabularnewline
\hline 
GPUs  & 8$\times $V100\tabularnewline
\hline 
Batch size per GPU  & 1\tabularnewline
\hline 
Acceleration factor  & 4\tabularnewline
\hline 
\end{tabular}
\quad
\begin{tabular}[t]{|c|c|}
\hline 
\textbf{Hyper-parameter} & \textbf{Value}\tabularnewline
\hline 
Low frequency lines  & 16\tabularnewline
\hline 
Mask type  & Offset-1\tabularnewline
\hline 
Seeds & 5\tabularnewline
\hline 
Decay & 0.0\tabularnewline
\hline 
$\beta_1, \beta_2$ & 0.9, 0.999 \tabularnewline
\hline 
\end{tabular}

\subsection{RCNN}
Our RCNN experiments use Detectron2\footnote{\url{https://github.com/facebookresearch/detectron2}}, and we use a pretrained ResNet backbone\footnote{ \texttt{detectron2://ImageNetPretrained/MSRA/R-50.pkl}}.

\begin{tabular}[t]{|c|c|}
\hline  
\textbf{Hyper-parameter}  & \textbf{Value}\tabularnewline
\hline
Backbone  & ResNet-50 \tabularnewline
\hline 
Max Iter & 200000 \tabularnewline
\hline
IMS Per Batch & 16 \tabularnewline
\hline
\end{tabular}
\quad
\begin{tabular}[t]{|c|c|}
\hline 
\textbf{Hyper-parameter}  & \textbf{Value}\tabularnewline
\hline 
GPUs  & 8$\times $V100\tabularnewline
\hline
Momentum & 0.9 \tabularnewline
\hline 
Decay & 1.5e-4 \tabularnewline
\hline 
\end{tabular}

\section{Popularity of standard learning rate schedules}
\begin{tabular}{|c|c|}
\hline 
\textbf{PyTorch Scheduler} & \textbf{Github Files (in thousands)}\tabularnewline
\hline 
\hline 
ReduceLROnPlateau & 105\tabularnewline
\hline 
StepLR & 101\tabularnewline
\hline 
MultiStepLR & 37.9\tabularnewline
\hline 
CosineAnnealingLR & 37.1\tabularnewline
\hline 
ExponentialLR & 16\tabularnewline
\hline 
OneCycleLR & 14.9\tabularnewline
\hline 
CosineAnnealingWarmRestarts & 10.9\tabularnewline
\hline 
CyclicLR & 9.1\tabularnewline
\hline 
LinearLR & 5.9\tabularnewline
\hline 
ConstantLR & 3.6\tabularnewline
\hline 
MultiplicativeLR & 2.6\tabularnewline
\hline 
PolynomialLR & 1.3\tabularnewline
\hline 
\end{tabular}
\label{table:sched-popularity}


\newpage
\section*{NeurIPS Paper Checklist}

\begin{enumerate}

\item {\bf Claims}
    \item[] Question: Do the main claims made in the abstract and introduction accurately reflect the paper's contributions and scope?
    \item[] Answer: \answerYes{}
    \item[] Justification: Our abstract makes clear theoretical and practical claims which are precisely addressed in the body of the paper.

\item {\bf Limitations}
    \item[] Question: Does the paper discuss the limitations of the work performed by the authors?
    \item[] Answer: \answerYes{}
    \item[] Justification: We specifically have a limitations section (Section~\ref{sec:limitations}) which addresses the main limitations of our work.

\item {\bf Theory Assumptions and Proofs}
    \item[] Question: For each theoretical result, does the paper provide the full set of assumptions and a complete (and correct) proof?
    \item[] Answer: \answerYes{}
    \item[] Justification: Proofs of all theoretical results are included in the Appendix.

    \item {\bf Experimental Result Reproducibility}
    \item[] Question: Does the paper fully disclose all the information needed to reproduce the main experimental results of the paper to the extent that it affects the main claims and/or conclusions of the paper (regardless of whether the code and data are provided or not)?
    \item[] Answer: \answerYes{}
    \item[] Justification: We include extensive details about all experiments conducted in the paper. Section~\ref{sec:experiment_details} includes hyper-parameter tables that allow for precise replication of all results.

\item {\bf Open access to data and code}
    \item[] Question: Does the paper provide open access to the data and code, with sufficient instructions to faithfully reproduce the main experimental results, as described in supplemental material?
    \item[] Answer: \answerYes{}
    \item[] Justification: We have already made source code available online for our refinement procedure. Our experiments rely on existing official open source code bases, with small modifications to use other learning rate schedules.

\item {\bf Experimental Setting/Details}
    \item[] Question: Does the paper specify all the training and test details (e.g., data splits, hyperparameters, how they were chosen, type of optimizer, etc.) necessary to understand the results?
    \item[] Answer:  \answerYes{}
    \item[] Justification: We did our best to specify all relevant experimental parameters in Section~\ref{sec:experiment_details}. In most cases, the hyper-parameters used closely match the original paper describing the method being tested 

\item {\bf Experiment Statistical Significance}
    \item[] Question: Does the paper report error bars suitably and correctly defined or other appropriate information about the statistical significance of the experiments?
    \item[] Answer: \answerYes{}
    \item[] Justification: We include error bars showing 2 standard-error of the mean (SEM) levels all our experiments, with the exception of the large-language-model pretraining experiments where only a single training run is feasible.

\item {\bf Experiments Compute Resources}
    \item[] Question: For each experiment, does the paper provide sufficient information on the computer resources (type of compute workers, memory, time of execution) needed to reproduce the experiments?
    \item[] Answer: \answerNo{}
    \item[] Justification: In all of our experiments, we are running exiting code-bases from previous papers, in which the compute costs of the methods are already described. Our approach adds no additional measurable overhead.
    
\item {\bf Code Of Ethics}
    \item[] Question: Does the research conducted in the paper conform, in every respect, with the NeurIPS Code of Ethics \url{https://neurips.cc/public/EthicsGuidelines}?
    \item[] Answer: \answerYes{}
    \item[] Justification: Optimization research doesn't touch upon any of the major concerns raised in the code of ethics.

\item {\bf Broader Impacts}
    \item[] Question: Does the paper discuss both potential positive societal impacts and negative societal impacts of the work performed?
    \item[] Answer: \answerNA{}
    \item[] Justification: There are no direct societal consequences from our work, beyond general concerns that apply to all work in the field.
    
\item {\bf Safeguards}
    \item[] Question: Does the paper describe safeguards that have been put in place for responsible release of data or models that have a high risk for misuse (e.g., pretrained language models, image generators, or scraped datasets)?
    \item[] Answer: \answerNA{}.
    \item[] Justification: We are not releasing any data or models.

\item {\bf Licenses for existing assets}
    \item[] Question: Are the creators or original owners of assets (e.g., code, data, models), used in the paper, properly credited and are the license and terms of use explicitly mentioned and properly respected?
    \item[] Answer: \answerNo{}
    \item[] Justification: All figures included in the paper are our own. Since we do not release any data or models, we do not discuss licensing terms of the datasets used in this work. We provide references in all cases to further information, including licensing, for each dataset.

\item {\bf New Assets}
    \item[] Question: Are new assets introduced in the paper well documented and is the documentation provided alongside the assets?
    \item[] Answer: \answerNA{}.
    \item[] Justification: There are no new assets.

\item {\bf Crowdsourcing and Research with Human Subjects}
    \item[] Question: For crowdsourcing experiments and research with human subjects, does the paper include the full text of instructions given to participants and screenshots, if applicable, as well as details about compensation (if any)? 
    \item[] Answer: \answerNA{}.
    \item[] Justification: Not applicable.

\item {\bf Institutional Review Board (IRB) Approvals or Equivalent for Research with Human Subjects}
    \item[] Question: Does the paper describe potential risks incurred by study participants, whether such risks were disclosed to the subjects, and whether Institutional Review Board (IRB) approvals (or an equivalent approval/review based on the requirements of your country or institution) were obtained?
    \item[] Answer: \answerNA{}
    \item[] Justification: Not applicable.

\end{enumerate}

\end{document}